\documentclass[10pt,journal,compsoc]{IEEEtran}
\ifCLASSOPTIONcompsoc
  \usepackage[nocompress]{cite}
\else
  \usepackage{cite}
\fi
\ifCLASSINFOpdf
\else
\fi

\usepackage{amsmath,amsfonts,bm}

\def\eqref#1{equation~\ref{#1}}

\def\1{\bm{1}}

\def\rvm{{\mathbf{m}}}

\def\rvw{{\mathbf{w}}}

\def\rmE{{\mathbf{E}}}

\def\rmM{{\mathbf{M}}}

\def\rmW{{\mathbf{W}}}

\def\vmu{{\bm{\mu}}}

\def\vh{{\bm{h}}}

\def\vm{{\bm{m}}}

\def\vw{{\bm{w}}}
\def\vx{{\bm{x}}}

\DeclareMathAlphabet{\mathsfit}{\encodingdefault}{\sfdefault}{m}{sl}
\SetMathAlphabet{\mathsfit}{bold}{\encodingdefault}{\sfdefault}{bx}{n}

\def\gC{{\mathcal{C}}}
\def\gD{{\mathcal{D}}}

\def\gL{{\mathcal{L}}}
\def\gM{{\mathcal{M}}}

\newcommand{\R}{\mathbb{R}}

\DeclareMathOperator*{\argmax}{arg\,max}
\DeclareMathOperator*{\argmin}{arg\,min}

\usepackage[colorlinks,citecolor=blue,urlcolor=blue]{hyperref}
\usepackage{url}
\usepackage{color,soul}
\usepackage{multirow}
\usepackage{epsfig}
\usepackage{booktabs}
\usepackage{subcaption}
\usepackage{amsthm}

\newtheorem{theorem}{Theorem}
\newtheorem*{theorem*}{Theorem}

\newtheorem{definition}{Definition}

\usepackage[table]{xcolor}
\definecolor{highlightgreen}{HTML}{39b54a}
\newcommand{\highlightnumber}[1]{
	\textcolor{purple}{\textbf{#1}}
}
\usepackage{multirow}

\begin{document}
%

\title{Neural Collapse Terminus: A Unified Solution for Class Incremental Learning and Its Variants}

%
%
%
%

\author{
        Yibo Yang$^*$, 
        Haobo Yuan$^*$, 
        Xiangtai Li, 
        Jianlong Wu, 
        Lefei Zhang,
        Zhouchen Lin,~\IEEEmembership{Fellow,~IEEE}, \\
        Philip H.S. Torr, 
        Dacheng Tao,~\IEEEmembership{Fellow,~IEEE},
        Bernard Ghanem
        
\IEEEcompsocitemizethanks{
\IEEEcompsocthanksitem Y. Yang and B. Ghanem are wth King Abdullah University of Science and Technology, Jeddah, Saudi Arabia.
\IEEEcompsocthanksitem H. Yuan and L. Zhang are with School of Computer Science, Wuhan University, Wuhan, China. $^*$ denotes equal contribution. 
\IEEEcompsocthanksitem X. Li and Z. Lin are with National Key Lab. of General Artificial Intelligence, School of Intelligence Science and Technology, Peking University, Beijing, China.
\IEEEcompsocthanksitem J. Wu is with Harbin Institute of Technology (Shenzhen), China.
\IEEEcompsocthanksitem P. Torr is with University of Oxford, Oxford, United Kindom.    
\IEEEcompsocthanksitem D. Tao is with University of Sydney, Sydney, Australia.


}
}

\IEEEtitleabstractindextext{%
\begin{abstract}

Class incremental learning (CIL) seeks to learn a model with new classes being encountered incrementally. How to enable learnability for new classes while keeping the capability well on old classes has been a crucial challenge for this task. Beyond the normal case, long-tail class incremental learning (LTCIL) and few-shot class incremental learning (FSCIL) are also proposed to consider the data imbalance and data scarcity, respectively, which are common in real-world implementations and further exacerbate the well-known problem of catastrophic forgetting. Existing methods are specifically proposed for one of the three tasks, trying to induce enough margin between the new-class and old-class prototypes and prevent a drastic drift of backbone features when each new session comes in. However, this would inevitably cause a misalignment between the feature and prototype of old classes, which explains the origin of catastrophic forgetting. In this paper, inspired by neural collapse that reveals an optimal feature-classifier geometric structure in classification, we offer a unified solution to the misalignment dilemma in the three tasks. Concretely, we propose neural collapse terminus that is a fixed structure with the maximal equiangular inter-class separation for the whole label space. It serves as a consistent target throughout the incremental training to avoid dividing the feature space incrementally. For CIL and LTCIL, we further propose a prototype evolving scheme to drive the backbone features into our neural collapse terminus smoothly. 
Our method also works for FSCIL with only minor adaptations. Theoretical analysis indicates that our method holds the neural collapse optimality in an incremental fashion regardless of data imbalance or data scarcity. We also design a generalized case where we do not know the total number of classes and whether the data distribution is normal, long-tail, or few-shot for each coming session, to test the generalizability of our method. 
Extensive experiments with multiple datasets are conducted to demonstrate the effectiveness of our unified solution to all the three tasks and the generalized case. Code address: \url{https://github.com/NeuralCollapseApplications/UniCIL}.
\end{abstract}

}

\maketitle

\IEEEdisplaynontitleabstractindextext

%
\IEEEpeerreviewmaketitle

\IEEEraisesectionheading{\section{Introduction}\label{sec:introduction}}

%
%
%
%

\IEEEPARstart{H}{uman} intelligence can easily acquire new knowledge in an incremental manner even with limited data. For artificial intelligence models, however, despite the great success of deep learning in a closed label space~\cite{lecun2015deep}, it is still challenging to learn new classes continually without forgetting the capability on old classes \cite{cauwenberghs2000incremental,polikar2001learn++}. To this end, \textit{class incremental learning} (CIL)~\cite{mensink2013distance,li2017learning,rebuffi2017icarl} was proposed to tackle this problem. CIL demands a single model to have a satisfactory classification performance on all the classes that have been encountered in the training of an incrementally appeared data stream~\cite{rebuffi2017icarl}. In real-world implementations, data imbalance and data scarcity are common, and in many applications, the two demands - learning incrementally and learning with imbalanced or insufficient data - emerge simultaneously. As a result, beyond the normal case, \emph{long-tail class incremental learning} (LTCIL) \cite{liu2022long} was proposed to extra consider data imbalance, and as an extreme case, \emph{few-shot class incremental learning} (FSCIL) \cite{tao2020few} was also proposed where there are only very limited (usually 5) samples for each new class in the incremental sessions. 




The key challenge shared by the three problems (CIL, LTCIL, and FSCIL) is the infamous \textit{catastrophic forgetting}~\cite{goodfellow2013empirical,rebuffi2017icarl}, which means that deep neural networks tend to forget the acquired knowledge drastically when updated with new data. For CIL, there have been intense studies ~\cite{liu2020mnemonics, zhao2020maintaining, iscen2020memory, yu2020semantic, hu2021distilling, yan2021dynamically, liu2021adaptive, wu2021striking, pernici2021class, zhou2021co, liu2021rmm, ni2021alleviate, liu2022model, shi2022mimicking, wang2022foster, pourkeshavarzi2022looking, zhu2021class, bhunia2022doodle,ashok2022class,simon2021learning} trying to achieve two important objectives: 1) when training on novel classes, making the backbone network stable to weaken the feature shift so as to prevent from losing the capability on old classes~\cite{hou2019learning, douillard2020podnet, liu2021adaptive, liu2022model}; 2) inducing a large margin between the newly learned novel-class prototypes (class prototypes refer to classifier vectors in this paper) and the previously learned old-class ones to enable the recognition on new classes~\cite{hou2019learning, ni2021alleviate, shi2022mimicking, wang2022foster}. However, as shown in Figure~\ref{fig1}-(a), since the old-class prototypes are already aligned with their corresponding features, adjusting them to be separated from the new-class prototypes will drive them shifted. Considering that the backbone is usually intentionally kept stable~\cite{rebuffi2017icarl, hou2019learning, douillard2020podnet}, the shift will inevitably cause a misalignment between the backbone features and the class prototypes.

\begin{figure*}[!t]
	\centering
	\includegraphics[width=0.75\linewidth]{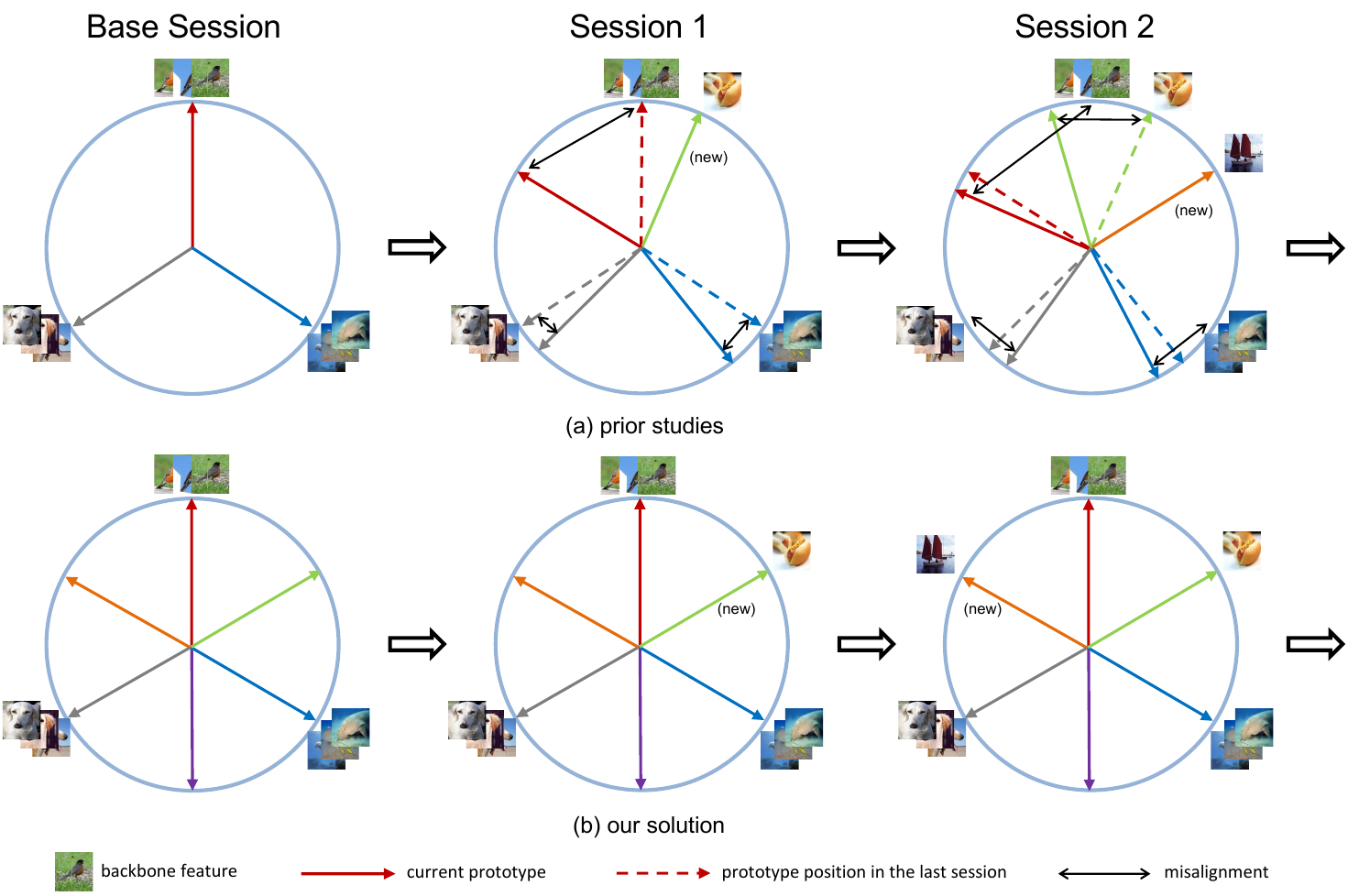}
	\caption{A sketch comparison between prior studies and our solution. 
	(a) Prior studies evolve the old-class prototypes via delicate loss or regularizer to keep them separated from the new-class ones, but will cause misalignment with the backbone features that are usually kept stable (and even fixed in FSCIL). 
	(b) Our solution pre-divides the feature space as a consistent target 
	and trains a model towards the same optimality to avoid dividing the feature space incrementally. }
	\label{fig1}
\end{figure*}

LTCIL \cite{liu2022long} relies on combining the CIL method with imbalanced learning strategies. For FSCIL \cite{tao2020few,dong2021few}, the problem becomes harder because few-shot data in novel sessions can easily induce over-fitting \cite{vinyals2016matching,ravi2017optimization,snell2017prototypical,sung2018learning}. Hence, existing studies favor training a backbone network on the base session as a feature extractor \cite{zhang2021few,hersche2022constrained,akyurek2021subspace}. For novel sessions, the backbone network is fixed and only a set of new-class prototypes are learned continually. But the newly added prototypes may lie close to the old-class ones, which impedes the ability to discriminate between the old-class and new-class samples. As a result, sophisticated loss functions or regularizers are usually adopted to keep a sufficient margin between the old-class and the new-class prototypes \cite{chen2021incremental,hersche2022constrained,akyurek2021subspace}. However, similar to CIL, there will be a misalignment between the adjusted class prototypes and the fixed backbone features for old classes. 
A recent study proposes to reserve feature space for novel classes to circumvent their conflict with old classes \cite{zhou2022forward}, but still an optimal feature-classifier alignment is hard to be guaranteed with a learnable classifier \cite{biondi2023cores}.



In summary, existing studies mainly propose a method for one of the three tasks. Based on the aforementioned analysis, we point out that all the catastrophic forgetting problems in CIL, LTCIL, and FSCIL can be largely attributed to the same origin - \textit{misalignment dilemma}, which refers to the misalignment between backbone features and class prototypes on old classes during training a model incrementally. To this end, we pose and study the following question, 


\emph{``Can we look for and pre-assign an optimal feature-classifier alignment such that a model is optimized towards the same fixed optimality throughout the incremental training as a unified solution to CIL, LTCIL, and FSCIL?"}




\subsection{Motivations and Contributions}

The feature-classifier alignment can be explored from a geometric view. An elegant phenomenon discovered in \cite{papyan2020prevalence} named \emph{neural collapse} indicates that, as a classification model is trained towards convergence (after 0 training error rate), the last-layer features output from a backbone network will be collapsed into their within-class feature means, and the feature means of all different classes will be aligned with their corresponding class prototypes in a form of simplex \emph{equiangular tight frame} (ETF), which is a special geometric structure such that $K$ vectors in a high dimension $d\ge K-1$ has equiangular separation with the minimized cosine similarity of $-\frac{1}{K-1}$. 
Particularly when $d=K-1$, a simplex ETF reduces to a regular simplex such as triangle and tetrahedron. 
The features as instructed by neural collapse have the minimized within-class variance due to the collapse, and the maximized between-class variance due to the ETF structure \cite{martinez2001pca}. So, it can be deemed as an optimal state with the maximized discriminant ratio \cite{fisher1936use,rao1948utilization} for classification. Following studies aim to theoretically explain this phenomenon \cite{fang2021exploring,han2022neural}. Interestingly, this structure is also proved to be the global optimality of supervised contrastive loss~\cite{graf2021dissecting}, which is similar to the widely used inter-class separation loss proposed in LUCIR~\cite{hou2019learning}.

However, imperfect training conditions, such as imbalanced learning \cite{fang2021exploring,yang2022we,xie2022neural}, will destruct the emergence of neural collapse and cause deteriorated performance. Similarly, when the feature space is divided incrementally in incremental learning, the appealing symmetric structure in neural collapse cannot emerge naturally as well, let alone in the harder cases, LTCIL and FSCIL, that combine both difficulties. This motivates us to keep such an optimal state instructed by neural collapse as sound as possible in incremental learning tasks. Concretely, we propose \emph{neural collapse terminus} that is a pre-assigned simplex equiangular tight frame structure for the whole label space. As shown in Figure~\ref{fig1}-(b), it serves as a consistent target and keeps invariant throughout the incremental training. 
In each session, what we need to do is to align the backbone feature towards its corresponding terminus to avoid assigning the feature space incrementally. As a result, there is no need to optimize the discrimination between novel and old classes via delicate regularization \cite{chen2021incremental,hersche2022constrained,akyurek2021subspace}. 

A preliminary version of our study has appeared as a spotlight presentation in ICLR 2023 \cite{yang2023neural}. However, the method in \cite{yang2023neural} only deals with FSCIL, and \textbf{cannot} be applied into CIL or LTCIL directly. In FSCIL, the backbone is usually intentionally not learnable on the few-shot data of incremental sessions, while in CIL and LTCIL, one can finetune the backbone with knowledge distillation and data rehearsal \cite{li2017learning,hou2019learning,rebuffi2017icarl}. Directly using the fixed ETF classifier in \cite{yang2023neural} will induce a large loss on new classes and thus cause a drastic shift of the backbone features to forget on old classes. Inherited from \cite{yang2023neural}, in this paper, we further develop a prototype evolving scheme named \emph{flying to collapse}, which gradually evolves the prototypes from the nearest class means into the neural collapse terminus. It enables to attain the pre-assigned feature-classifier alignment in a smooth manner to facilitate the training of backbone features towards the neural collapse terminus for CIL and LTCIL. 

In this way, our method can be a unified solution to the three tasks.
On CIL and LTCIL, our architecture and training setting are exactly the same, and on FSCIL, we only need minor adaptations due to different data access and training strategies allowed. 
Concretely, we apply the misalignment loss in \cite{yang2022we} into both aligning and distillation for CIL and LTCIL, and introduce a projection layer and a feature mean memory for FSCIL. 
In order to further test the generalizability, we design a generalized case, where the total number of classes and the data distribution of each new session are unknown. 
Thanks to the interpretability of neural collapse, we also perform theoretical analysis that shows 
the global optimality using our method satisfies neural collapse regardless of training incrementally or data imbalance. 
We conduct extensive experiments with multiple datasets to demonstrate the effectiveness of our unified method on all the three tasks and the generalized case.

In summary, the contributions of this paper can be listed as follows:

\begin{itemize}
	\item We point out that the catastrophic forgetting in multiple class incremental learning tasks derive from the same origin, the misalignment dilemma. Based on neural collapse, we propose a unified solution named neural collapse terminus as a consistent target throughout the incremental training to relieve the feature-classifier misalignment. 
	As far as we know, we are the first to propose a unified solution to the three class incremental learning tasks.
	\item For CIL and LTCIL, we propose a prototype evolving scheme named flying to collapse that
	drives the backbone features smoothly into the neural collapse terminus to avoid sharp shift. 
	We also apply a novel misalignment loss for both aligning and distillation. 
	For FSCIL, we introduce a projection layer and a feature mean memory to align the features when the backbone network is fixed.
	\item We perform theoretical analysis to show that our method can induce the neural collapse optimality regardless of incremental training or data imbalance. 
	\item To further test the generalizability of our method, we also design a generalized case where we do not know the total number of classes and whether each session is normal, imbalanced, or few-shot. Extensive experiments with multiple datasets are conducted with state-of-the-art performances on all the three tasks and the generalized case. 
\end{itemize}








\section{Related Work}
\label{Related_work}

\subsection{Class Incremental Learning and Its Variants}

\subsubsection{The normal case}
Different from task incremental learning~\cite{aljundi2018memory, lopez2017gradient, mirzadeh2020understanding, van2022three}, in which the task (session) information is available at test time, class incremental learning (CIL) learns classes sequentially and infers data without knowing which session it comes from. 
To mitigate the catastrophic forgetting~\cite{rebuffi2017icarl,goodfellow2013empirical} in CIL, several kinds of strategies have been proposed.

\noindent
\textbf{Knowledge Distillation.} It is natural to leverage knowledge distillation (KD)~\cite{hinton2014distilling} with the model from the earlier session as the teacher network to alleviate forgetting when training on each incremental session~\cite{li2016learning}. Some studies use multiscale based KD~\cite{zhou2019m2kd, douillard2020podnet}, class-logit based KD~\cite{rebuffi2017icarl, castro2018end, wu2019large, zhao2020maintaining}, or cosine-similarity constraint~\cite{hou2019learning} to make the backbone stable on old-class features. However, it is hard for KD to ensure an optimal tradeoff between the stability on old classes and  the plasticity for new classes \cite{douillard2020podnet, wu2021striking}. 


\noindent
\textbf{Rehearsal.} The rehearsal strategy refers to letting the model have access to a memory with a limited size for saving data and labels (exemplars) of old classes. This strategy is widely adopted by CIL methods~\cite{hou2019learning, douillard2020podnet, liu2021adaptive, liu2021rmm, shi2022mimicking} since iCaRL~\cite{rebuffi2017icarl}. Beyond the basic memory allocation strategy (\emph{i.e.} \textit{herding}~\cite{rebuffi2017icarl, hou2019learning}), following works resort to generating exemplars~\cite{shin2017continual, liu2020mnemonics} or using REINFORCE~\cite{williams1992simple} to select informative exemplars of different classes~\cite{liu2021rmm}.

\noindent
\textbf{Dynamic Network Architecture.} Using an ``incremental architecture'' for CIL is good for preserving old knowledge as well as leaving enough model capacity for new data~\cite{douillard2022dytox, li2021preserving, yan2021dynamically, wang2022foster}. Some other works freeze partial network parameters for similar motivation~\cite{liu2021adaptive, abati2020conditional}. However, dynamic architecture may induce unbounded computational and memory overhead. Also, fixing partial parameters may introduce noise for new classes, as noticed in~\cite{wang2022foster}.

\noindent
\textbf{Regularization.} 
Existing methods~\cite{rebuffi2017icarl, hou2019learning, douillard2020podnet} introduce regularization terms to consolidate the previous knowledge and induce enough margin for incoming classes. Advanced 
designs include preserving the topology of features~\cite{tao2020topology} and mimicking the oracle model at the initial session~\cite{shi2022mimicking}. Nonetheless, the two goals are hard to attain simultaneously without qualification via regularization. 

In this work, we use the neural collapse inspired terminus as a consistent target throughout the incremental training to tackle the misalignment dilemma without introducing any advanced regularization term or dynamic architecture. Following \cite{hou2019learning}, our method for CIL only adopts the basic KD~\cite{hou2019learning} and the basic {herding} rehearsal strategy~\cite{rebuffi2017icarl} for simplicity, but still achieves superior performance over the methods with complex strategies.

\subsubsection{LTCIL and FSCIL}

CIL has several variants that incorporate the challenges in imbalanced learning \cite{liu2022long}, few-shot learning \cite{tao2020few,dong2021few}, and federated learning \cite{dong2022federated}. In this paper, apart from the normal CIL, we mainly focus on the long-tail case LTCIL, and the few-shot case FSCIL. In \cite{liu2022long}, the popular two-stage methods for long-tail classification \cite{cao2019learning,chu2020feature, kang2019decoupling, zhang2021distribution, zhong2021improving} are combined with the typical CIL methods \cite{castro2018end, hou2019learning, douillard2020podnet} as a baseline for LTCIL. First proposed by \cite{tao2020few}, FSCIL only has a few novel-classes and training data in each incremental session \cite{tao2020few,dong2021few}, which increases the tendency of overfitting on novel classes \cite{snell2017prototypical,sung2018learning}. This requires a careful balance between well adapting to novel classes and less forgetting on old classes \cite{zhao2021mgsvf}. Some studies try to make base and incremental sessions compatible via pseudo-feature \cite{cheraghian2021synthesized,zhou2022forward}, augmentation \cite{peng2022few}, or looking for a flat minima \cite{shi2021overcoming}. Another choice is to use meta learning \cite{yoon2020xtarnet,chi2022metafscil,zhou2022few}. Similar to CIL, the old-class prototypes should have enough margin with the new-class ones, and meanwhile, the shifted old-class prototypes should not be too far away from their original positions where the fixed backbone features are located. 
Current studies widely rely on evolving the prototypes \cite{zhang2021few,zhu2021self} or sophisticated designs of loss and regularizer \cite{ren2019incremental,hou2019learning,tao2020topology,joseph2022energy,lu2022geometer,hersche2022constrained,chen2021incremental,akyurek2021subspace,yang2022dynamic}. 
However, the two goals have inherent conflict, and a tough effort to balance the loss terms is necessary. \emph{In contrast, our method for FSCIL only uses a single loss without any regularizer.} 


\subsection{Neural Collapse}

Neural collapse first discovered in~\cite{papyan2020prevalence} reveals an elegant phenomenon about the last-layer backbone features and class prototypes in a well-trained neural network. Beyond this observation, later studies theoretically explain neural collapse by proving this phenomenon as the global optimality of balanced training with the cross entropy~\cite{weinan2022emergence, lu2020neural, graf2021dissecting, fang2021exploring, zhu2021geometric, ji2021unconstrained} and mean squared error~\cite{mixon2020neural, poggio2020explicit, zhou2022optimization, han2022neural, tirer2022extended} loss functions. More recent studies have shown that neural collapse can be a good perspective to study imbalanced learning~\cite{yang2022we, xie2022neural, thrampoulidis2022imbalance,behnia2023implicit, zhong2023understanding}, transfer learning~\cite{galanti2022on,li2022principled}, generalization \cite{xu2023dynamics}, and the power of quantum neural network \cite{du2022demystify}. To the best of our knowledge, \emph{we are the first to study a unified solution to multiple class incremental learning tasks from the neural collapse perspective, which offers our method sound interpretability}.



\section{Preliminaries}
\label{background}

\subsection{Problem Setup}


\noindent\textbf{The normal case:}
Class Incremental Learning (CIL) trains a neural network model progressively in a sequence of datasets $\{{\mathcal{D}}^{(0)}, {\mathcal{D}}^{(1)}, \dots, {\mathcal{D}}^{(T)}\}$, where ${\mathcal{D}}^{(t)} = \{(\vx_i, y_i)\}_{i=1}^{|{\mathcal{D}}^{(t)}|}$, and the label space of the session $t$ is $\mathcal{C}^{(t)} = {\rm set}(\{y_i|(\vx_i,y_i)\in\mathcal{D}^{(t)}\})$. Note that in CIL, $\mathcal{C}^{(t)}\cap\mathcal{C}^{(t')}=\emptyset$, $\forall t'\ne t$, which means that there is no overlap of class labels between sessions. Usually, there are plenty of classes in $\{{\mathcal{D}}^{(0)}\}$ for a model to get the basic visual capability, and we call $\{{\mathcal{D}}^{(0)}\}$ as \textit{base session}. The training on $\{{\mathcal{D}}^{(1)}, {\mathcal{D}}^{(2)}, \dots, {\mathcal{D}}^{(T)}\}$ is called \textit{incremental sessions}. After training of each session, the model is tested on the unified label space $\mathcal{\hat{C}}^{(t)} = \cup_{i=0}^{t}\mathcal{C}^{(i)}$, measuring the performance on test data of all the classes that have been encountered. 

\noindent\textbf{The long-tail case:}
In LTCIL, the number of training samples follows an exponential decay across all classes. 
The remaining settings are the same as CIL. 

\noindent\textbf{The few-shot case:}
In FSCIL, each incremental session $\gD^{(t)}$, $t>0$, only has a few labeled images and we have $|\gD^{(t)}|=pq$, where $p$ is the number of classes and $q$ is the number samples per novel class, known as $p$-way $q$-shot.

\noindent\textbf{Experimental settings:} The widely adopted practice for CIL allows the usage of knowledge distillation (KD) from the model trained in the previous session, and a limited size of memory to store the data from the earlier sessions (rehearsal) \cite{hou2019learning}. Some studies adopt more strategies such as dynamic architecture \cite{liu2021adaptive,yan2021dynamically} or complex regularization \cite{shi2022mimicking,zhou2022forward}. The setup for LTCIL is the same as CIL. But for FSCIL, usually only the data in $\gD^{(t)}$ is accessible in session $t$,
and the training sets of the previous sessions are not available to store. A widely adopted remedy is to store the feature mean of each old class \cite{cheraghian2021semantic, chen2021incremental, akyurek2021subspace, hersche2022constrained}. 
In our method, we only adopt the basic allowable strategies, KD and rehearsal for CIL and LTCIL, and storing feature mean for FSCIL, as detailed in Section \ref{implement}. 




\subsection{Neural Collapse Background}



\begin{definition}[Simplex Equiangular Tight Frame]
	\label{ETF}
	A simplex equiangular tight frame (ETF) refers to a collection of vectors $\{\mathbf{e}_i\}_{i=1}^K$ in $\mathbb{R}^d$, $d\ge K-1$, that satisfies:
	\begin{equation}\label{mimj}
	\mathbf{e}^T_{k_1}\mathbf{e}_{k_{2}}=\frac{K}{K-1}\delta_{k_1,k_2}-\frac{1}{K-1},\ \ \forall k_1, k_2\in[1,K],
    \end{equation}
    where $\delta_{k_1,k_2}=1$ when $k_1=k_2$, and 0 otherwise. All vectors have the same $\ell_2$ norm and any pair of two different vectors has the same inner product of $-\frac{1}{K-1}$, which is the minimum possible cosine similarity for $K$ equiangular vectors in $\mathbb{R}^d$. 
\end{definition}

A simplex equiangular tight frame can be produced from an orthogonal basis via:
	\begin{equation}\label{ETF_M}
		\rmE=\sqrt{\frac{K}{K-1}}\mathbf{U}\left(\mathbf{I}_K-\frac{1}{K}\mathbf{1}_K\mathbf{1}_K^T\right),
	\end{equation}
where $\mathbf{E}=[\mathbf{e}_1,\cdots,\mathbf{e}_K]\in\mathbb{R}^{d\times K}$ is a simplex ETF, 
$\mathbf{U}\in\mathbb{R}^{d\times K}$ is an orthogonal basis and satisfies $\mathbf{U}^T\mathbf{U}=\mathbf{I}_K$,
$\mathbf{I}_K$ is an identity matrix, and $\mathbf{1}_K$ is an all-ones vector.

Then the neural collapse phenomenon can be formally described by the following four properties \cite{papyan2020prevalence}:

\textbf{(NC1)}: The last-layer features of the same class will collapse into their within-class mean, \emph{i.e.,} the covariance $\Sigma^{(k)}_W\rightarrow\mathbf{0}$, where $\Sigma^{(k)}_W=\mathrm{Avg}_{i}\{(\vmu_{k,i}-\vmu_{k})(\vmu_{k,i}-\vmu_{k})^T\}$, $\vmu_{k,i}$ is the feature of sample $i$ in class $k$, and $\vmu_{k}$ is the within-class feature mean of class $k$;

\textbf{(NC2)}: The feature means of all classes centered by the global mean will converge to the vertices of a simplex ETF defined by Definition~\ref{ETF}, \emph{i.e.,} $\{\hat{\vmu}_k\}_{k=1}^K$, satisfy Eq. (\ref{mimj}), where $\hat{\vmu}_k=(\vmu_k - \vmu_G) / \|  \vmu_k - \vmu_G \|$ and $ \vmu_G$ is the global mean;

\textbf{(NC3)}: The feature means centered by the global mean will be aligned with their corresponding class prototypes (classifier vector), which means that the class prototypes will converge to the same simplex ETF, \emph{i.e.,} $\hat{\vmu}_k=\vw_k/ \| \vw_k \|$, $1\le k \le K$, where $\vw_k$ is the class prototype of class $k$;

\textbf{(NC4)}: When \textbf{(NC1)}-\textbf{(NC3)} hold, the model prediction based on logits can be simplified to select the nearest class center\footnote{We omit the bias term in a linear classifier layer for simplicity.}, \emph{i.e.,} $\argmax_k\langle\vmu, \vw_k\rangle=\argmin_k||\vmu-\vmu_k||$, where $\langle\cdot\rangle$ is the inner product operator, $\vmu$ is the last-layer feature of a sample for prediction.

Due to the feature invariance within each class and the maximal equiangular separation among all class centers, neural collapse describes an optimal feature-classifier alignment status for classification.



\section{Method}
\label{methods}



 In Section \ref{etf_classifier}, we propose a consistent training target for class incremental learning tasks, based on which we develop our methods for CIL, LTCIL, and FSCIL in Sections \ref{method_CIL} - \ref{method_FSCIL}, and conduct theoretical support in Section \ref{theoretical}. In Section \ref{generalized_case}, we introduce a generalized case. The implementations are detailed in Section \ref{implement}.

\subsection{A Consistent Target---Neural Collapse Terminus}
\label{etf_classifier}

The neural collapse phenomenon cannot naturally happen in imperfect training conditions, such as imbalanced learning. In incremental training, at each incremental session $t$, the label space will be enlarged from $\mathcal{\hat{C}}^{(t - 1)}$ to $\mathcal{\hat{C}}^{(t)}$, and the classifier needs to update from $\mathbf{W}_{t-1}\in\mathbb{R}^{d\times |\mathcal{\hat{C}}^{(t - 1)}|}$ to $\mathbf{W}_{t}\in\mathbb{R}^{d\times |\mathcal{\hat{C}}^{(t)}|}$ to match the extended label space, where $d$ is the dimension of the backbone feature. When the new-class prototypes for $\mathcal{C}^{(t)}$ are merged with $\mathbf{W}_{t-1}$, the distance between the old-class and new-class prototypes may be close. It is hard to ensure an equal separation among all classes via regularization. Besides, the adjusted old-class prototypes may cause misalignment with their backbone features, which could be the main underlying reason for the catastrophic forgetting problem \cite{joseph2022energy}. Therefore, the motivation of our unified solution is to induce neural collapse for 
an optimal feature-classifier alignment and keep it as sound as possible during incremental training.


Suppose that for a class incremental learning problem the base session contains a label space of $K_0$ classes, each incremental session has $p$ classes, and we have $T$ incremental sessions in total. The whole label space has $K_0+K'$ classes, where $K'=Tp$, \emph{i.e.,} we need to learn a model that can recognize samples from $K_0+K'$ classes. 
In order to alleviate the misalignment dilemma, 
we propose to pre-assign an optimal status, named neural collapse terminus, and train a model towards this consistent target throughout incremental training. This can be achieved via fixing class prototypes as the structure instructed by neural collapse. 
Actually, using a fixed classifier has been proved successful in classification without sacrificing performance \cite{hoffer2018fix,pernici2021regular,tanwisuth2021prototype}. 
We randomly initialize class prototypes $\hat{\rmW}_{\rm ETF}\in\R^{d\times (K_0+K')}$ by Eq. (\ref{ETF_M}) for the whole label space, \emph{i.e.,} the union of classes in all sessions $\cup_{t=0}^{T}\gC^{(t)}$, where $K_0=|\gC^{(0)}|$ and $K'=\sum^T_{t=1}|\gC^{(t)}|=Tp$. Here we assume that we know the total number of classes in advance. In Section \ref{generalized_case}, we will discuss how to apply our method when this prior is unknown. 
Then any pair ($k_1,k_2$) of class prototypes in $\hat{\rmW}_{\rm ETF}$ satisfies:
\begin{equation}\label{k1k2}
	\hat{\mathbf{w}}^T_{k_1}\hat{{\rvw}}_{k_{2}}=\frac{K_0+K'}{K_0+K'-1}\delta_{k_1,k_2}-\frac{1}{K_0+K'-1},
\end{equation}
for all $k_1, k_2\in[1,K_0+K']$, where $\hat{\rvw}_{k_1}$ and $\hat{\rvw}_{k_2}$ are two column vectors in $\hat{\rmW}_{\rm ETF}$. 

$\hat{\rmW}_{\rm ETF}$ ensures that the whole label space has the maximal equiangular separation, \emph{i.e.,} any two prototypes of different classes have the same cosine similarity of $-\frac{1}{K-1}$. It is not learnable, and serves as a consistent target for an incrementally trained model to attain. In this way, there is no need to assign the feature space incrementally, and shift prototypes via sophisticated regularization. What we need to do in incremental training is to align the backbone features into their corresponding prototypes.




\subsection{Normal Class Incremental Learning (CIL)} 

\label{method_CIL}

\subsubsection{Flying to collapse}
\label{cil-fly}

We have introduced neural collapse terminus as a consistent target. 
However, how to facilitate training towards this ideal terminus is not a trivial problem.
At the beginning of each incremental session, the new-class features have a large misalignment with their target prototypes. If we directly minimize the distance between backbone features and the terminus, new classes will contribute much more to the loss value and dominate the backbone training. As a result, it will cause the forgetting on old classes in another way.

To pave the way to neural collapse terminus (NCT), we further propose \textit{Flying to Collapse} (FtC) that leads the backbone features from a starting point into the terminus smoothly to avoid their sharp shift. 
The starting point should induce a lower misalignment between the new-class features and prototypes at the beginning of each session's training, such that the backbone will not be dominated by the new classes. To this end, we adopt the  \textit{nearest class mean} (NCM)~\cite{mensink2013distance} as the initial class prototypes for novel classes, which can be formulated as:
\begin{equation}
	\label{ncm}
	\mathbf{w}^{(NCM)}_c={\rm Avg}_i\{\hat{\vmu}_i | y_i=c, c\in \mathcal{C}^{(t)}\},
\end{equation}
where $c$ is a novel class, $\hat{\vmu}_i={\vmu}_i/\|{\vmu}_i\|$, ${\vmu}_i=f(\vx_i, \theta_{f})$ is the initial backbone feature of input $\vx_i$, and $y_i$ is its label. A series of works~\cite{rebuffi2017icarl, snell2017prototypical, laenen2021on, guo2022learning} use NCM as the classifier to infer test data because it has a close average distance to the features of the same class.



During the training of each incremental session, we gradually evolve the class prototypes from NCM to NCT:
\begin{equation}
	\label{fly}
	\mathbf{w}_c = \eta \hat{\mathbf{w}}_c^{(NCT)} + (1 - \eta) \hat{\mathbf{w}}_c^{(NCM)},\ \eta = \frac{e}{E},
\end{equation}
where $\hat{\mathbf{w}}_c^{(NCM)}$ is the NCM prototype in Eq. (\ref{ncm}) after $\ell_2$-normalization, $\hat{\mathbf{w}}_c^{(NCT)}$ is the prototype in our $\hat{\rmW}_{\rm ETF}$ for class $c$, $e$ is the epoch index, and $E$ is the number of total epochs used to train this session. 
$\eta$ is gradually evolved from 0 to 1 as training goes on. So, at the beginning of each incremental session, we have $\eta=0$, and $\mathbf{w}_c = \hat{\mathbf{w}}^{(NCM)}_c$. At the end of training, $\mathbf{w}_c = \hat{\mathbf{w}}_c^{(NCT)}$, which corresponds to the pre-assigned target instructed by neural collapse. The evolving process is illustrated in Figure~{\ref{fig3}}.

In this way, the novel-class prototypes start from a position close to their features and then drive the features smoothly into the desired terminus. Each epoch's training only advances a mild step to avoid a sharp shift of backbone features that will cause the forgetting on old classes.

\subsubsection{Training}
\label{cil-training}

We denote a backbone network as $f$, and then we have $\vmu=f(\vx, \theta_{f})$, where $\vmu\in\R^{d}$ is the output feature of input $\vx$, and $\theta_f$ is the backbone network parameters.


The aim of our training is to drive the backbone features towards their corresponding class prototypes. 
When our training reaches its optimality, the features of all classes should be aligned with NCT after the final session, which satisfies neural collapse and gets rid of feature-classifier misalignment. 
Because knowledge distillation (KD) is enabled for CIL, another goal is to make the backbone features close between two sequential sessions. 
Therefore, our loss functions in each session's training are composed of two terms, the align term and the distillation term. 

We use a misalignment loss to align the backbone features with their prototypes,
which can be formulated as:
\begin{equation}
	\label{dr}
	\mathcal{L}_{\rm{align}}\left(\hat{\vmu}_i,\hat{\mathbf{w}}_{y_i}\right)=\frac{1}{2}\left(\hat{\mathbf{w}}^T_{y_i} \hat{\vmu}_i - 1\right)^2,
\end{equation}
where $\hat{\vmu}_i$ is the $\ell_2$-normalized backbone feature of the $i$-th sample that could be a novel-class data or an old-class exemplar, 
$y_i$ is its class label, 
and $\hat{\mathbf{w}}_{y_i}$ is the $\ell_2$-normalized class prototype for $y_i$ and evolves by Eq. (\ref{fly}) during training. 
Because $\|\hat{\mathbf{w}}_{y_i}\|=\|\hat{\vmu}_i\|=1$, it is easy to identify that the loss in Eq. (\ref{dr}) reaches its optimality ($\mathcal{L}_{\rm{align}}$=0) if and only if $\cos(\mathbf{w}_{y_i}, {\vmu}_i)=1$, which means that the feature is aligned with its corresponding class prototype.

With the basic knowledge distillation technique in \cite{hou2019learning}, we can store the backbone network trained in the previous session. 
The distillation loss is in a similar form as:
\begin{equation}
	\label{dr_dist}
	\mathcal{L}_{\rm{distill}}\left(\hat{\vmu}_i^{(t-1)},\hat{\vmu}_i^{(t)}\right)=\frac{1}{2}\left(\left(\hat{\vmu}_i^{(t-1)}\right)^T \hat{\vmu}_i^{(t)} - 1\right)^2,
\end{equation}
where $\hat{\vmu}_i^{(t-1)}$ is the feature generated by the old backbone network in the session $(t - 1)$. Our method does not require any other regularization, so the final loss for CIL can be formulated as:
\begin{equation}
	\label{loss_final}
	\mathcal{L} = \mathcal{L}_{\rm{align}} + \lambda \cdot \mathcal{L}_{\rm{distill}}.
\end{equation}

In evaluation, suppose that a feature extracted by the backbone is $\vmu \in \mathbb{R}^d$. We adopt cosine similarity between $\vmu$ and the target $\hat{\rmW}_{\rm ETF}$ as the classification criterion, \emph{i.e.,} $\argmax_k\cos\angle(\vmu, \hat{\mathbf{w}}_k)$, where $\hat{\rvw}_k$ is the column vector in $\hat{\rmW}_{\rm ETF}$ for class $k$. 
Note that cosine similarity criterion has been widely adopted in CIL~\cite{hou2019learning} and other computer vision tasks~\cite{gidaris2018dynamic, peng2022few, wang2018cosface}. It is able to eliminate the inclination towards novel classes caused by learnable bias or modulus~\cite{hou2019learning}.

\begin{figure}[t]
	\centering
	\includegraphics[width=1\linewidth]{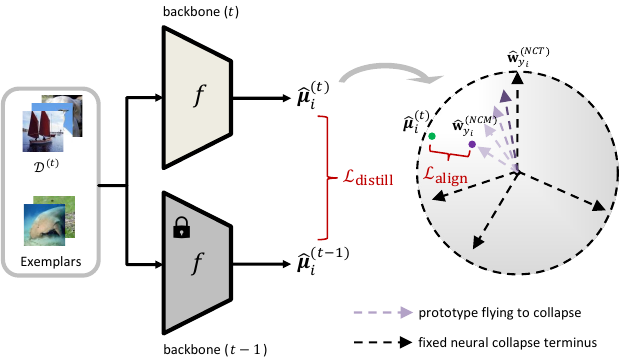}
	\caption{An illustration of our architecture for CIL and LTCIL. The input includes the training data $\gD^{(t)}$ of the current session $t$, and the exemplars allowable for CIL and LTCIL. The arrows moving from light to dark represent the prototypes using our flying to collapse strategy. It starts from the nearest class mean $\hat{\mathbf{w}}_{y_i}^{(NCM)}$ and terminates at our neural collapse terminus $\hat{\mathbf{w}}_{y_i}^{(NCT)}$ by Eq. (\ref{fly}). We adopt this simple pipeline for both CIL and LTCIL without bells and whistles.}
	\label{fig3}
\end{figure}

\subsection{Long-tail Class Incremental Learning (LTCIL)}
\label{method_LTCIL}

Long-tail class incremental learning (LTCIL) is recently proposed in~\cite{liu2022long} to incorporate the common imbalanced data distribution in the real world to CIL, and is thus a more challenging task. In LTCIL~\cite{liu2022long}, the number of training samples follows an exponential decay across all classes that could be ordered or shuffled. 
The other settings are the same as CIL. The current solution \cite{liu2022long} to LTCIL is combining the popular two-stage method in long-tail recognition~\cite{cao2019learning,chu2020feature, kang2019decoupling, zhang2021distribution, zhong2021improving} with the typical CIL methods~\cite{castro2018end, hou2019learning, douillard2020podnet}. That is to say, for each session, a model is trained with an instance-balanced sampler in the first stage, and then only the classifier is re-trained with a class-balanced sampler in the second stage.




However, the two-stage long-tail recognition methods also rely on a learnable classifier. 
As indicated by \cite{fang2021exploring,yang2022we,biondi2023cores}, neural collapse is also vulnerable to imbalanced training data using a learnable classifier. Therefore, our proposed NCT with fixed class prototypes is intrinsically suitable for LTCIL. 
It is natural to expect that our method can be \textbf{directly} applied to LTCIL even \textbf{without} the two-stage training strategy or any other modification. 
In experiments, 
we train for CIL and LTCIL using the same method and training settings to show the effectiveness on CIL and its direct extensibility to LTCIL. 
It is noteworthy that our method is the first one-stage solution to the LTCIL task. 

\subsection{Few-shot Class Incremental Learning (FSCIL)}
\label{method_FSCIL}

\begin{figure*}[t]
	\centering
	\includegraphics[width=0.83\linewidth]{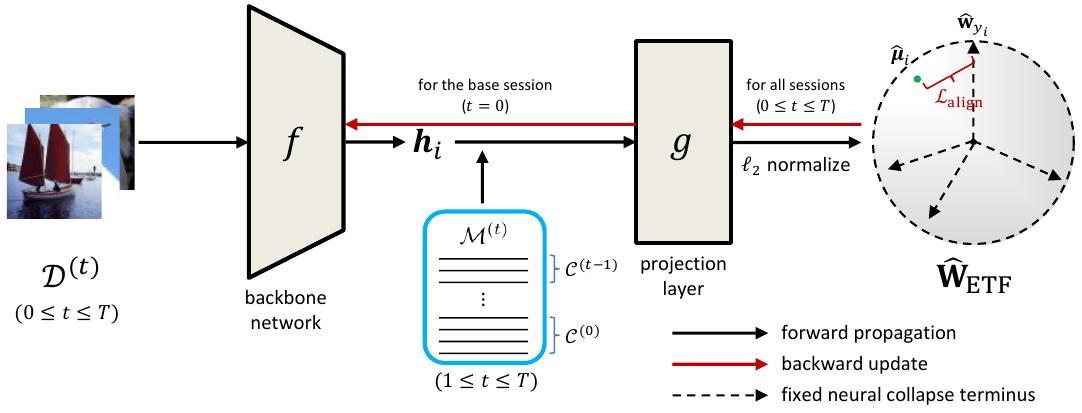}
	\caption{An illustration of our NC-FSCIL. $\vh_i$ is the intermediate feature from the backbone network $f$. $\hat{\vmu}_i$ is the normalized output feature after the projection layer $g$. $\hat{\rmW}_{\rm ETF}$ is the neural collapse terminus that contains the prototypes of the whole label space and serves as a consistent target throughout the incremental training. $\gL_{\rm align}$ denotes the misalignment loss function. $f$ is frozen in the incremental sessions ($1\le t\le T$). $\gM^{(t)}$ is a small memory of old-class feature means that is widely adopted in prior studies such as \cite{cheraghian2021semantic}, \cite{chen2021incremental}, \cite{akyurek2021subspace}, and \cite{hersche2022constrained}.}
	\label{fig2}
\end{figure*}


In FSCIL, only few-shot data is available for incremental sessions and storing old-class data is not allowed. Thus, we need to slightly adjust the architecture and training based on our method for CIL and LTCIL. 
As shown in Figure~\ref{fig2}, 
we append a projection layer $g$ after the backbone network $f$.
The backbone network $f$ takes the training data $\vx_i$ as input, and outputs an intermediate feature $\vh_i$. The projection layer $g$ can be a linear transformation or an MLP block following \cite{hersche2022constrained,peng2022few}. It projects the intermediate feature $\vh_i$ into $\vmu_i$. 
Finally, 
we get the $\ell_2$-normalized output feature $\hat{\vmu}_i$, \emph{i.e.,}
\begin{equation}
	\label{muhat}
	\hat{\vmu}_i={\vmu}_i/\|{\vmu}_i\|,\quad \vmu_i=g(\vh_i, \theta_{g}),\quad \vh_i=f(\vx_i, \theta_{f}),
\end{equation}
where $\theta_f$ and $\theta_g$ refer to the parameters of the backbone network and the projection layer, respectively. 

We use $\hat{\vmu}_i$ to compute error signal by Eq. (\ref{dr}). Different from CIL and LTCIL, the class prototypes for FSCIL do not evolve with our flying to collapse, \emph{i.e.,} $\eta=1$ in Eq. (\ref{fly}).
In the base session $t=0$, we jointly train both $f$ and $g$ using the base session data. The empirical risk to minimize in the base session can be formulated as:
\begin{align}\label{min_base}
	\min_{\theta_f, \theta_g}\quad &\frac{1}{|\gD^{(0)}|}\sum_{(\vx_i,y_i)\in\gD^{(0)}}\gL_{\rm{align}}\left(\hat{\vmu}_i,\hat{\mathbf{w}}_{y_i}\right),
\end{align}
where $\hat{\mathbf{w}}_{y_i}$ is the class prototype in $\hat{\rmW}_{\rm ETF}$ for class $y_i$,
as introduced in Section \ref{etf_classifier}, 
$\gL_{\rm{align}}$ is the misalignment loss as defined in Eq. (\ref{dr}), 
and $\hat{\vmu}_i$ is a function of $f$ and $g$ as Eq. (\ref{muhat}).

In each incremental session $1\le t\le T$, we fix the backbone network $f$ as a feature extractor, and only finetune the projection layer $g$. 
As a widely adopted practice in FSCIL studies, a memory of old-class features can be retained to relieve the overfitting on novel classes \cite{cheraghian2021semantic,chen2021incremental,akyurek2021subspace,hersche2022constrained}. Some studies store the feature means as class prototypes for inference instead \cite{zhou2022forward,zhou2022few,peng2022few}. 
Following \cite{hersche2022constrained}, we only keep a memory $\gM^{(t)}$ of the intermediate feature mean $\vh_c$ for each old class $c$. Concretely, we have, 
\begin{equation}
	\gM^{(t)}=\{\vh_c| c\in \cup_{j=0}^{t-1}\gC^{(j)}\},\  \vh_c = {\rm Avg}_i \{ f(\vx_i, \theta_{f}) | y_i = c\}, 
\end{equation}
for all $1\le t\le T$, where $f$ has been fixed after the base session. Then we use $\gD^{(t)}$ as the input of $f$, and $\gM^{(t)}$ as the input of $g$ to finetune the projection layer $g$. The empirical risk to minimize in incremental sessions is formulated as:
\begin{align}\label{min_inc}
	\min_{\theta_g}\quad &\frac{1}{|\gD^{(t)}|+|\gM^{(t)}|}\left( \gL_{(\gD^{(t)} )} +\gL_{(\gM^{(t)})} \right),
\end{align}
where
$	 \gL_{(\gD^{(t)} )} = \sum_{(\vx_i,y_i)\in\gD^{(t)}}\gL_{\rm{align}}\left(\hat{\vmu}_i,\hat{\mathbf{w}}_{y_i}\right), 
	 \gL_{(\gM^{(t)})} = \sum_{(\vh_c, c)\in\gM^{(t)}}\gL_{\rm{align}}\left(\hat{\vmu}_c,\hat{\mathbf{w}}_{c}\right),
$
$\hat{\vmu}_i$ and $\hat{\vmu}_c$ are the output features of $\vx_i$ and $\vh_c$, respectively, $|\gD^{(t)}|$ is the number of training samples in session $t$, and we have $|\gM^{(t)}|=\sum^{t-1}_{j=0}|\gC^{(j)}|$. 

Therefore, we solve Eq. (\ref{min_base}) on the base session $t=0$, and solve Eq. (\ref{min_inc}) on the incremental sessions $1\le t\le T$. 
\textbf{Thanks to our pre-assigned consistent target, we do not rely on any regularizer in our training}. The evaluation for FSCIL is in the same way as CIL and LTCIL. 



\begin{table*}[!t]
	\renewcommand\arraystretch{1.}
	\small
    \centering
    \caption{Comparison of {average incremental accuracy} on CIFAR-100 and ImageNet-100 datasets. $B$ denotes the number of classes in the base session. The number of exemplars for each class is 20 (except for RMM, which uses an advanced rehearsal strategy and allows a more relaxed memory budget). 
    RMM~\cite{liu2021rmm} is integrated on LUCIR-AANets baseline. CwD~\cite{shi2022mimicking} is with AANet and PODNet baselines (both stronger than LUCIR). CSCCT~\cite{ashok2022class} is with PODNet baseline. 
    Our method is integrated on the plain LUCIR baseline. ``-'' indicates not available.}
     \setlength{\tabcolsep}{10pt}
    \begin{tabular}{lc|ccc|ccc}
        \toprule
        \multirow{2}{*}{Method} & \multirow{2}{*}{Venue} & \multicolumn{3}{c}{CIFAR-100 ($B$=50)} & \multicolumn{3}{c}{ImageNet-100 ($B$=50)}\\
        && 5 steps & 10 steps & 25 steps & 5 steps & 10 steps &25 steps\\
        \midrule
        LwF~\cite{li2016learning}&\scriptsize{(ECCV16)} & 53.59 & 48.66 & 45.56 & 53.62 & 47.64 & 44.32\\

        iCaRL~\cite{rebuffi2017icarl}&\scriptsize{(CVPR17)} &  60.82 & 53.74 & 47.86 & 65.44 & 59.88 & 52.97\\
        
        LUCIR~\cite{hou2019learning}&\scriptsize{(CVPR19)} & 66.27 & 60.80 & 52.96 & 70.60 & 67.76 & 62.76\\

        PODNET~\cite{douillard2020podnet}&\scriptsize{(ECCV20)} & 66.98 & 63.76 & 61.00 & 75.71 & 72.80 & 65.57\\

        AANet~\cite{liu2021adaptive}&\scriptsize{(CVPR21)} & 69.79 & 67.97 & 64.92 & 71.96 & 70.05 & 67.28\\

        DER~\cite{yan2021dynamically}&\scriptsize{(CVPR21)} & 72.60 & 72.45 & - & - & \textbf{77.73} & -\\

        RMM~\cite{liu2021rmm}&\scriptsize{(NeurIPS21)} & 68.42 & 67.17 & 64.56 & 73.58 & 72.83 & 72.30\\

        FOSTER~\cite{wang2022foster}&\scriptsize{(ECCV22)} & - & 67.95 & 63.83 & - & 77.54 & 69.34\\

        CSCCT~\cite{ashok2022class}&\scriptsize{(ECCV22)} & - & 63.72 & 61.10 & 76.41 & 74.35 & 68.91\\

        AFC~\cite{kang2022class}&\scriptsize{(CVPR22)} & 66.49 & 64.98 & 63.89  & 76.87 & 75.75 & 73.34\\
        
        CwD~\cite{shi2022mimicking}&\scriptsize{(CVPR22)} & 70.30 & 68.62 & 66.17 & 76.91 & 74.34 & 68.18\\
        \midrule
        \textbf{NC-CIL~(Ours)} &-& \textbf{75.28} & \textbf{73.33} & \textbf{68.35} & \textbf{79.29} & 77.62 & \textbf{74.97}\\
        \bottomrule
    \end{tabular}
    \label{tab:com}
\end{table*}

\subsection{Theoretical Supports}
\label{theoretical}
We perform our theoretical work based on a simplified model that drops the backbone network and only keeps the last-layer features and class prototypes as independent variables to optimize. The rationality lies in that current backbone networks are usually over-parameterized and can produce features along any direction. 
This simplification has been widely adopted in prior studies to facilitate analysis \cite{graf2021dissecting,fang2021exploring,zhu2021geometric}. We investigate the optimality of an incremental problem of $T$ sessions with our $\hat{\rmW}_{\rm ETF}$. Concretely, we consider the following problem,
\begin{align}
	\label{obj}
	\min_{\rmM^{(t)}}\quad & \frac{1}{N^{(t)}}\sum^{K^{(t)}}_{k=1}\sum^{n_k}_{i=1} \gL\left(\rvm^{(t)}_{k,i},\hat{\rmW}_{\rm ETF}\right), \ 0\le t \le T, \\
	s.t. \quad & \| \rvm^{(t)}_{k,i} \|^2 \le 1, \quad \forall 1\le k \le K^{(t)},\ 1\le i \le n_k, \notag
\end{align}
where $\rvm^{(t)}_{k,i}\in\R^d$ denotes a feature variable that belongs to the $i$-th sample of class $k$ in session $t$, $n_k$ is the number of samples in class $k$, 
$K^{(t)}$ is the number of classes in session $t$, 
$N^{(t)}$ is the number of samples in session $t$, \emph{i.e.,} $N^{(t)}=\sum_{k=1}^{K^{(t)}}n_k$, and $\rmM^{(t)}\in\R^{d\times N^{(t)}}$ denotes a collection of $\vm^{(t)}_{k,i}$. $\hat{\rmW}_{\rm ETF}\in\R^{d\times K}$ refers to the fixed consistent target for the whole label space as introduced in Section \ref{etf_classifier}, and we have $K=\sum_{t=0}^TK^{(t)}$. 
$\gL$ can be both the traditional cross entropy (CE) loss and the misalignment loss that we adopt. 
\begin{theorem}
	\label{theorem}
	Let $\hat{\rmM}^{(t)}$ denotes the global minimizer of Eq. (\ref{obj}) by optimizing the model incrementally from $t=0$, and we have $\hat{\rmM}=[\hat{\rmM}^{(0)}, \cdots, \hat{\rmM}^{(T)}]\in\R^{d\times \sum_{t=0}^T N^{(t)}}$. No matter if $\gL$ in Eq. (\ref{obj}) is CE or misalignment loss, for any column vector $\hat{\rvm}_{k,i}$ in $\hat{\rmM}$ whose class label is $k$, we have:
	\begin{equation}
		\label{theorem_eq}
		\|\hat{\rvm}_{k,i}\|=1,\ \ \hat{\rvm}^T_{k,i}\hat{\rvw}_{k'}=\frac{K}{K-1}\delta_{k,k'}-\frac{1}{K-1},
	\end{equation}
	 for all $k, k'\in[1,K],\ 1\le i \le n_k$, where $K=\sum_{t=0}^TK^{(t)}$ denotes the total number of classes of the whole label space, $\delta_{k,k'}=1$ when $k=k'$ and 0 otherwise, and $\hat{\rvw}_{k'}$ is the class prototype in $\hat{\rmW}_{\rm ETF}$ for class $k'$. 
\end{theorem}
The proof of Theorem \ref{theorem} can be found in the Appendix. Eq. (\ref{theorem_eq}) indicates that the global minimizer $\hat{\rmM}$ of Eq. (\ref{obj}) reaches the neural collapse terminus, \emph{i.e.,} features of the same class collapse into a single vertex, and the vertices of all classes are aligned with $\hat{\rmW}_{\rm ETF}$ as a simplex ETF. 
More importantly, in problem Eq. (\ref{obj}), the number of classes $K^{(t)}$ for all $T+1$ sessions and the number of samples $n_k$ for all $K$ classes can be distributed in any way,
which corresponds to the challenging demands of LTCIL and FSCIL.

\subsection{A Generalized Case}
\label{generalized_case}


There are two implicit prerequisites for our methods. In Section \ref{etf_classifier}, our neural collapse terminus is based on the assumption that we know the number of total classes in advance. Our architecture for FSCIL slightly differs from the one for CIL and LTCIL, which means we also need to know whether the data distribution is normal, long-tail, or few-shot in advance. 
However, in real-world implementations, we may not know the total class number and the data distribution of the next session.
In this subsection, we deal with a generalized case where the two priors are unknown and show how to apply our method accordingly. 

When we cannot know the total class number, our method is still able to work well. We only need to initialize the neural collapse terminus with a very large number of class prototypes, denoted as $K_M$, such that the classes to be encountered in reality, denoted as $K$, are not possible to surpass $K_M$. The $K_M$ prototypes also have equiangular separation. The difference is that the pair-wise cosine similarity is increased from $-\frac{1}{K-1}$ to $-\frac{1}{K_M-1}$. When $K_M\rightarrow\infty$, the simplex ETF structure reduces to an orthogonal frame, whose margin is still enough to serve as a consistent target for incremental training. Actually, $-\frac{1}{K-1}$ and $-\frac{1}{K_M-1}$ are both close to $0$. Therefore, it almost does not affect the performance. In experiments, we pre-assign a neural collapse terminus with 1000 prototypes, but only use 100 of them for training on CIFAR-100 to simulate this case.

When the data distribution of an incremental problem is not specified, \emph{i.e.,} the data for the next session could possibly be any of the normal, long-tail, or few-shot cases, we adopt our architecture designed for FSCIL as shown in Figure~\ref{fig2}. For each incremental session, if the number of training data is larger than the few-shot case (\emph{e.g., 5 samples per class}), we fix the projection layer $g$ and only finetune the backbone network $f$ minimizing Eq. (\ref{loss_final}). Otherwise, we fix the backbone network $f$ and only finetune the projection layer $g$ solving Eq. (\ref{min_inc}). 
In this way, our method can be a unified solution regardless of data distribution. 
To simulate this case, for each incremental session, we first sample from the three cases with equal probability, \emph{i.e.,}
the choices of normal, long-tail, or few-shot distribution are equally possible,
and then randomly sample some of the remaining classes that have not been encountered. 

\begin{figure*}[!t]
	\centering
	\includegraphics[width=.32\linewidth]{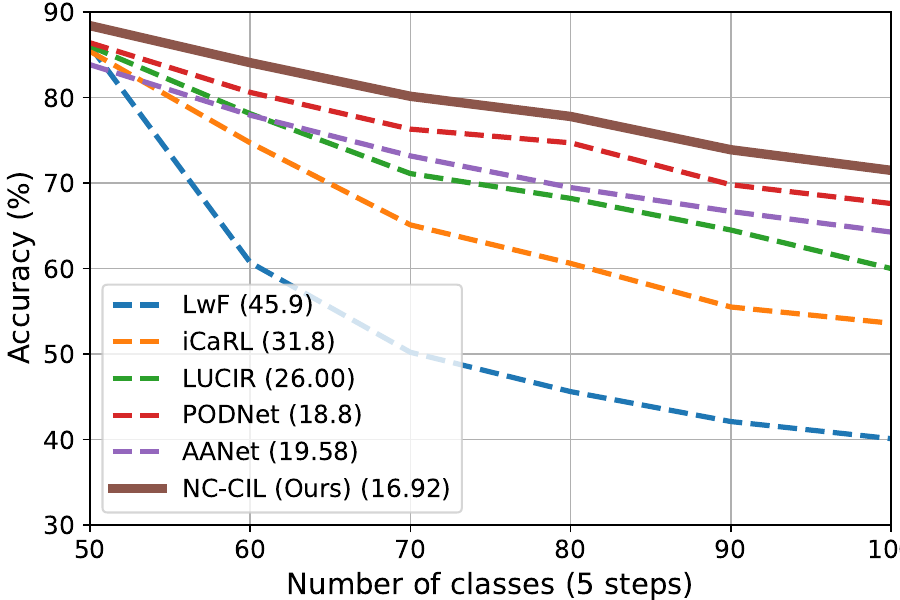}
	\includegraphics[width=.32\linewidth]{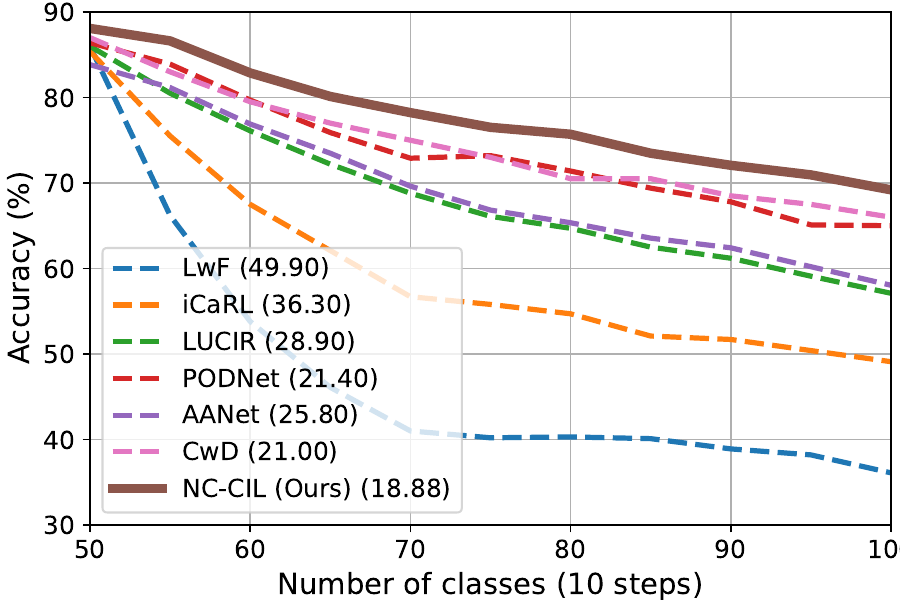}
	\includegraphics[width=.32\linewidth]{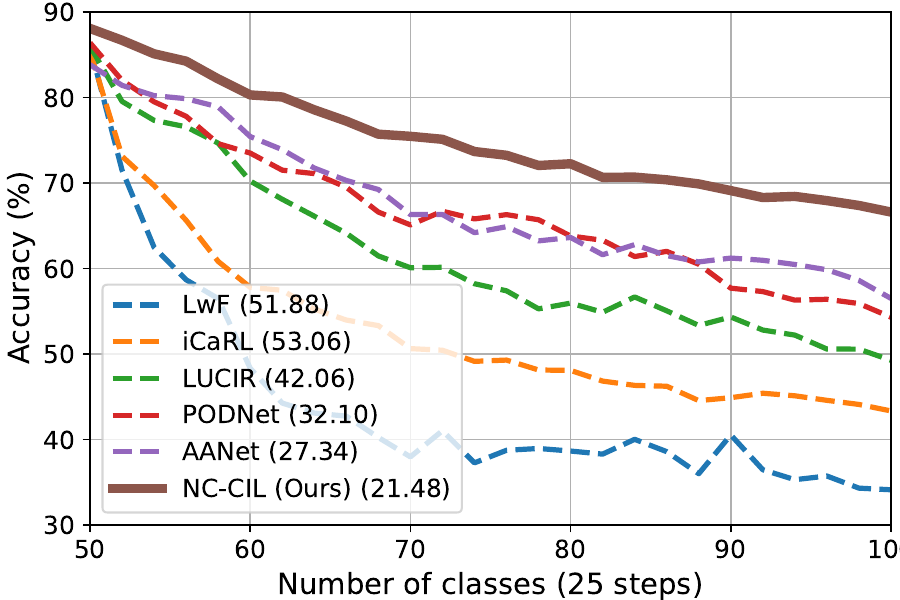}
	\caption{Comparison of the evaluation $A_t$ after each session on ImageNet-100. We compare our NC-CIL with LwF~\cite{li2016learning}, iCaRL~\cite{rebuffi2017icarl}, LUCIR~\cite{hou2019learning}, PODNet~\cite{douillard2020podnet}, AANet~\cite{liu2021adaptive}, and CwD~\cite{shi2022mimicking}. The results except for CwD~\cite{shi2022mimicking} are reproduced by~\cite{liu2021adaptive}. CwD results are estimated from their original paper~\cite{shi2022mimicking}. The \textbf{performance drop} (PD) is reported in the legend, which is the difference in performance between the base session and the final incremental session. A lower PD indicates better holding the capability on old classes. }
\label{fig:num}
\end{figure*}

\subsection{Implementations}
\label{implement}

Our method has a simple framework mainly focused on the neural collapse terminus. 
In order to highlight the efficacy of our method, for CIL and LTCIL,
we adopt a simple pipeline as LUCIR~\cite{hou2019learning}, which is widely used as the baseline method~\cite{shi2022mimicking,ashok2022class,liu2021adaptive}. 
The whole pipeline does \textbf{NOT} require dynamic architecture~\cite{liu2021adaptive, yan2021dynamically}, pseudo CIL task training~\cite{liu2021rmm}, advanced rehearsal strategy~\cite{liu2021rmm}, or more regularization term~\cite{zhou2022forward,shi2022mimicking}. 
We only adopt the most basic KD and rehearsal strategy in LUCIR~\cite{hou2019learning} for simplicity, which is the most challenging setup among the CIL studies. 
For FSCIL, we only introduce a memory of the intermediate feature means, which is widely used in existing FSCIL studies \cite{cheraghian2021semantic,chen2021incremental,akyurek2021subspace,hersche2022constrained,zhou2022forward,zhou2022few,peng2022few}. For the generalized case, we enable all these strategies (KD, rehearsal, and feature mean memory) because it could be either of the three cases during training.

\section{Experiments}


Although our method is integrated on the simple pipeline of LUCIR \cite{hou2019learning}, we will still compare with the state-of-the-art methods of the three tasks. Our training details are described in the Appendix. 

\subsection{Datasets and Evaluation Metric}

We validate our method on the standard benchmarks. 
For CIL, CIFAR-100 \cite{krizhevsky2009learning} is a relatively smaller dataset that contains (50,000 (train) / 10,000 (val)) 32x32 images evenly distributed in 100 classes. ImageNet-100 is a sub-dataset that randomly chooses 100 classes from the original ImageNet-1k~\cite{deng2009imagenet} in the same resolution. We follow previous works~\cite{hou2019learning,douillard2020podnet,liu2021adaptive,shi2022mimicking} that use a fixed seed (1993) to extract exactly the same 100 classes. ImageNet-1k contains (about 1,300 (train) / 50 (val)) images per class, and is with 1,000 classes. LTCIL~\cite{liu2022long} creates CIFAR100-LT and ImageNet100-LT to consider the long-tail data distribution. Specifically, LTCIL considers two settings, including \textit{Ordered LTCIL} and \textit{Shuffled LTCIL}. In Ordered LTCIL, the number of images decreases monotonically from the base session to the final incremental session. By contrast, Shuffled LT-CIL shuffles all the classes, and each session may have varying degrees of imbalance, which means that it is more challenging~\cite{liu2022long}.
For FSCIL, 
miniImageNet \cite{russakovsky2015imagenet} is a variant of ImageNet with a size of $84\times84$. It has 100 classes with each class containing 500 images for training and 100 images for testing. CIFAR-100 has the same number of classes and images, and the image size is $32\times32$. CUB-200 \cite{wah2011caltech} is a dataset for fine-grained image classification containing 11,788 images of 200 classes in a resolution of $224\times224$. There are 5,994 images for training and 5,794 images for testing. We follow the standard experimental settings in FSCIL \cite{tao2020few,zhang2021few}. For both miniImageNet and CIFAR-100, the base session ($t=0$) contains 60 classes, and a 5-way 5-shot (5 classes and 5 images per class) problem is adopted for each of the 8 incremental sessions ($1\le t\le8$). For CUB-200, 100 classes are used in the base session, and there are 10 incremental sessions, each of which is 10-way 5-shot. 
{For the generalized case, we test on CIFAR-100.}

\begin{table}[!t]
		\small
		\renewcommand\arraystretch{1.}
    \centering
        \caption{Comparison of {average incremental accuracy} on ImageNet-1k. The base session has an equal number of classes to each incremental session (e.g., ``10 steps'' means 100 classes for base or each incremental session). The number of exemplars for each class is 20. The results of previous studies are reproduced by~\cite{shi2022mimicking}.}
             \setlength{\tabcolsep}{8pt}
    \begin{tabular}{l|cc}
        \toprule
        \multirow{2}{*}{Method} & \multicolumn{2}{c}{ImageNet-1k}\\

        & 10 steps & 20 steps\\

        \midrule

        LwF~\cite{li2016learning} & 40.86 & 27.72\\
        iCaRL~\cite{rebuffi2017icarl} & 49.56 & 42.61\\
        LUCIR~\cite{hou2019learning} & 56.40 & 52.75\\
        PODNet~\cite{douillard2020podnet} & 57.01 & 54.06\\
        AANet~\cite{liu2021adaptive} & 51.76 & 46.86\\
        CwD~\cite{shi2022mimicking} & 58.18 & 56.01\\
        \midrule
        \textbf{NC-CIL (Ours)} & \textbf{58.21} & \textbf{57.58} \\
        \bottomrule
    \end{tabular}
    \label{tab:inet}
\end{table}

\begin{table*}[!ht]
	\renewcommand\arraystretch{1.}
    \centering
\small
    \caption{Comparison of {average incremental accuracy} on CIFAR100-LT and ImageNet100-LT. The base session is with 50 classes. The number of exemplars for each class is 20. The results of LUCIR~\cite{hou2019learning} and PODNet~\cite{douillard2020podnet} are reproduced by~\cite{liu2022long}. LT-CIL~\cite{liu2022long} is a two-stage method with LUCIR and PODNet baseline. Our NC-LTCIL has the same architecture as NC-CIL \textbf{without} any change or hyper-parameter tuning. $\uparrow$ indicates improvement.}
    \setlength{\tabcolsep}{8pt}
    {
    \begin{tabular}{lc|cccc|cccc}
        \toprule[0.1em]
        \multirow{2}{*}{Method} & \multirow{2}{*}{Mode} & \multicolumn{4}{c}{CIFAR100-LT ($\rho=0.01$)} & \multicolumn{4}{c}{ImageNet100-LT ($\rho=0.01$)}\\

        && 5 steps & $\uparrow$ & 10 steps & $\uparrow$ & 5 steps & $\uparrow$ & 10 steps & $\uparrow$\\

        \midrule

        LUCIR~\cite{hou2019learning} & \multirow{4}{*}{{Ordered}} & 42.69 & {\highlightnumber{+14.00}} & 42.15 & {\highlightnumber{+16.54}} & 52.91 & {\highlightnumber{+10.89}} & 52.80 & {\highlightnumber{+10.88}}\\

        PODNet~\cite{douillard2020podnet} &  & 44.07 & {\highlightnumber{+12.62}} & 43.96 & {\highlightnumber{+14.73}} & 58.78 & {\highlightnumber{+5.02}} & 58.94 & {\highlightnumber{+4.74}}\\

        LT-CIL~\cite{liu2022long} && 45.88 & {\highlightnumber{+10.81}} & 45.73 & {\highlightnumber{+12.96}} & 58.82 & {\highlightnumber{+4.98}} & 59.09 & {\highlightnumber{+4.59}}\\

        \textbf{NC-LTCIL (Ours)} && \textbf{56.69} & - & \textbf{58.69} & - & \textbf{63.80} & - & \textbf{63.68} & -\\

        \midrule

        LUCIR~\cite{hou2019learning} & \multirow{4}{*}{{Shuflled}} & 35.09 & {\highlightnumber{+10.54}}  & 34.59 & {\highlightnumber{+13.57}} & 45.80 & {\highlightnumber{+6.67}} & 46.52 & {\highlightnumber{+9.65}}\\

        PODNet~\cite{douillard2020podnet} &  & 34.64 & {\highlightnumber{+10.99}} & 34.84 & {\highlightnumber{+13.32}} & 49.69 & {\highlightnumber{+2.78}} & 51.05 & {\highlightnumber{+5.12}}\\

        LT-CIL~\cite{liu2022long} && 39.40 & {\highlightnumber{+6.23}} & 39.00 & {\highlightnumber{+9.16}} & 52.08 & {\highlightnumber{+0.39}} & 52.60 & {\highlightnumber{+3.57}}\\

        \textbf{NC-LTCIL (Ours)} && \textbf{45.63} & - & \textbf{48.16} & - & \textbf{52.47} & - & \textbf{56.17} & -\\
        \bottomrule[0.1em]
    \end{tabular}
}
    \label{tab:ltcil}
\end{table*}

 \begin{table*}[t] 
	\renewcommand\arraystretch{1.}
	\begin{center}
		\caption{Performance of FSCIL on miniImageNet and comparison with other studies. The top rows list class-incremental learning and few-shot learning results implemented by \cite{tao2020few,zhang2021few} in the FSCIL setting. ``Average Acc.'' is the average incremental accuracy. ``Final Improv.'' calculates the improvement of our method after the last session over prior studies. The last row lists the improvements of our method over ALICE \cite{peng2022few}, which is a strong baseline in FSCIL. 
		}
		\resizebox{\textwidth}{!}{
			\begin{tabular}{lccccccccccl}
				\toprule
				\multicolumn{1}{l}{\multirow{2}{*}{\bf Methods}} & \multicolumn{9}{c}{\bf Accuracy in each session (\%) $\uparrow$} & \bf Average & \bf Final \\ 
				\cmidrule{2-10}
				& \bf 0   & \bf 1      & \bf 2      & \bf 3    & \bf 4     & \bf 5  & \bf 6     & \bf 7      &\bf 8   & \bf Acc. & \bf Improv.  \\ 
				\midrule
				iCaRL~\cite{rebuffi2017icarl} &   61.31&	46.32&	42.94&	37.63&	30.49&	24.00&	20.89&	18.80&	17.21&	33.29 & \bf +41.1 \\
				NCM~\cite{hou2019learning}         & 61.31	&47.80&	39.30&	31.90&	25.70&	21.40&	18.70&	17.20&	14.17&	30.83& \bf +44.14 \\
				D-Cosine~\cite{vinyals2016matching}    &70.37&	65.45&	61.41&	58.00&	54.81&	51.89&	49.10&	47.27&	45.63&	55.99& \bf +12.68\\
				\midrule
				TOPIC \cite{tao2020few} & 61.31	& 50.09	& 45.17	& 41.16	& 37.48	& 35.52	& 32.19	& 29.46	& 24.42	& 39.64 & \bf  +33.89\\
				IDLVQ \cite{chen2021incremental} & 64.77 &	59.87	& 55.93	& 52.62	& 49.88	& 47.55	& 44.83	& 43.14	& 41.84	& 51.16 & \bf +16.47 \\
				Self-promoted~\cite{zhu2021self}  &	61.45	&	63.80	&	59.53	&	55.53	&	52.50	&	49.60	&	46.69	&	43.79	&	41.92	&	52.76&	\bf +16.39 \\
				CEC~\cite{zhang2021few}     &  72.00	&66.83&	62.97	&59.43	&56.70	&53.73	&51.19	&49.24	&47.63	&57.75 & \bf +10.68 \\
				LIMIT \cite{zhou2022few}&  72.32	& 68.47	& 64.30	& 60.78	& 57.95	& 55.07	& 52.70	& 50.72	& 49.19	& 59.06 &  \bf +9.12 \\
				Regularizer \cite{akyurek2021subspace} &80.37	&74.68	&69.39	&65.51	&62.38	&59.03	&56.36	&53.95	&51.73&	63.71& \bf +6.58 \\
				MetaFSCIL \cite{chi2022metafscil} &72.04&	67.94&	63.77&	60.29&	57.58&	55.16&	52.90&	50.79&	49.19&	58.85& \bf +9.12 \\
				C-FSCIL \cite{hersche2022constrained} &76.40&	71.14&	66.46&	63.29&	60.42&	57.46&	54.78&	53.11&	51.41&	61.61& \bf +6.90  \\
				Data-free Replay \cite{liu2022few} &71.84	&67.12	&63.21	&59.77	&57.01	&53.95	&51.55	&49.52	&48.21	&58.02& \bf +10.10\\
				ALICE \cite{peng2022few} &80.60	&70.60	&67.40	&64.50	&62.50	&60.00	&57.80	&56.80	&55.70	&63.99& \bf +2.61 \\
				\midrule
				\bf NC-FSCIL (ours) &  \bf 84.02&	\bf 76.80&	\bf 72.00&	\bf 67.83&	\bf 66.35	&\bf 64.04&	\bf 61.46&	\bf 59.54&	\bf 58.31&	\bf 67.82&  \\
				\midrule
				\emph{Improvement over ALICE} & +3.42	& +6.20&	+4.60&	+3.33&	+3.85&	+4.04&	+3.66&	+2.74&	+2.61&	{\color{purple} \textbf{+3.83}}&\\
				\bottomrule
			\end{tabular}
		}
		\label{table:imgnet}
	\end{center}
\end{table*}

The existing CIL~\cite{rebuffi2017icarl, hou2019learning, liu2021adaptive, shi2022mimicking} and LT-CIL~\cite{liu2022long} studies have made a consensus on the evaluation metric. We adopt the same metric, \textit{average incremental accuracy}, to evaluate our proposed method. Suppose that a model is trained on one base session and $T$ incremental sessions. The average incremental accuracy is defined as:
\begin{equation}
	A = \frac{1}{T + 1}\sum_{t=0}^{T}A_t,
\end{equation}
where $A_t$ is the top-1 accuracy on the unified label space $\mathcal{\hat{C}}^{(t)} = \cup_{i=0}^{t}\mathcal{C}^{(i)}$ after the training of session $t$. For FSCIL, we report the accuracy $A_t$ after each session and also the average incremental accuracy following common practice.

\subsection{Performance on Benchmarks}

\noindent
\textbf{Results on CIL.} We compare our method with the other CIL methods on the CIFAR-100 and ImageNet-100 datasets in Table~\ref{tab:com}. The recent methods with strong performances focus on dynamic architecture (DER~\cite{yan2021dynamically}, FOSTER~\cite{wang2022foster}), advanced distillation strategy (AFC~\cite{kang2022class}, CSCCT~\cite{ashok2022class}), and advanced regularization terms (CwD~\cite{wang2022foster}). It is shown that our method outperforms these state-of-the-art methods without these advanced strategies. 
In particular, our method outperforms the LUCIR baseline by 8.7\%, 9.9\%, and 12.2\% on ImageNet-100 for 5/10/25 incremental-session settings, respectively. We also surpass the strong results achieved by CwD~\cite{wang2022foster} on both CIFAR-100 and ImageNet-100. To compare the accuracy of each incremental session, we also report the accuracy of each session along with the previous methods in Figure~\ref{fig:num}. 
Our method has the lowest performance drop, which measures the degree of forgetting in incremental training. 
Note that the original per-session accuracies of some methods~\cite{ashok2022class, kang2022class} are not available. On ImageNet-1k, as shown in Table~\ref{tab:inet}, our method still outperforms the LUCIR baseline and the recent CwD~\cite{shi2022mimicking}, whose method is also integrated on LUCIR.

\noindent
\textbf{Results on LTCIL.} For LTCIL experiments, we directly apply our method of CIL to LTCIL under the same training setting without any change, and compare with the state-of-the-art two-stage method proposed in~\cite{liu2022long}. 
Note that so far there have not been a lot of studies on LTCIL other than \cite{liu2022long} due to its difficulty. The imbalance ratio $\rho=0.01$ is very challenging -- the minimum class only has 5 samples for CIFAR100-LT and 13 for ImageNet100-LT. Nonetheless, 
our method outperforms the CIL methods and \cite{liu2022long} by a significant margin, as shown in Table~\ref{tab:ltcil}.
On CIFAR100-LT, we surpass all the compared methods by more than 10\% on the ordered mode, and 5\% on the shuffled mode.
On ImageNet100-LT, we also have a more salient advantage under the ordered mode. This may be because our method is more advantageous in alleviating the old-class forgetting due to the better feature-classifier alignment, which means enough training data in the earlier sessions will benefit our method to better keep the capability on old classes. 
We leave exploring this phenomenon and designing a more adaptive method or training strategy for future work to further advance the LTCIL task.

\begin{table*}[!ht] 
	\renewcommand\arraystretch{1.}
	\begin{center}
		\caption{Performance of FSCIL on CIFAR-100 and comparison with other studies. 
		}
		\resizebox{\textwidth}{!}{
			\begin{tabular}{lccccccccccl}
				\toprule
				\multicolumn{1}{l}{\multirow{2}{*}{\bf Methods}} & \multicolumn{9}{c}{\bf Accuracy in each session (\%) $\uparrow$} & \bf Average & \bf Final \\ 
				\cmidrule{2-10}
				& \bf 0   & \bf 1      & \bf 2      & \bf 3    & \bf 4     & \bf 5  & \bf 6     & \bf 7      &\bf 8   & \bf Acc. & \bf Improv.  \\ 
				\midrule
				iCaRL~\cite{rebuffi2017icarl} &   64.10	&53.28&	41.69&	34.13&	27.93&	25.06&	20.41&	15.48&	13.73&	32.87& \bf +42.38\\
				NCM~\cite{hou2019learning}        &64.10	&53.05	&43.96	&36.97	&31.61	&26.73	&21.23	&16.78	&13.54	&34.22& \bf +42.57\\
				D-Cosine~\cite{vinyals2016matching}    &74.55	&67.43	&63.63	&59.55	&56.11	&53.80	&51.68	&49.67	&47.68	&58.23& \bf +8.43\\
				\midrule
				TOPIC \cite{tao2020few} & 64.10	&55.88	&47.07	&45.16	&40.11	&36.38	&33.96	&31.55	&29.37&	42.62& \bf +26.74 \\
				
				Self-promoted~\cite{zhu2021self}  &	64.10	&65.86	&61.36	&57.45	&53.69	&50.75	&48.58	&45.66	&43.25	&54.52& \bf +12.86 \\
				
				CEC~\cite{zhang2021few}     & 73.07	&68.88	&65.26	&61.19	&58.09	&55.57	&53.22	&51.34	&49.14	&59.53& \bf +6.97\\
				DSN \cite{yang2022dynamic} & 73.00	&68.83	&64.82	&62.64	&59.36	&56.96	&54.04	&51.57	&50.00	&60.14& \bf +6.11 \\
				LIMIT \cite{zhou2022few}&  73.81	&72.09	&67.87	&63.89	&60.70	&57.77	&55.67	&53.52	&51.23	&61.84& \bf +4.88 \\
				
				MetaFSCIL \cite{chi2022metafscil} &74.50	&70.10	&66.84	&62.77	&59.48	&56.52	&54.36	&52.56	&49.97	&60.79& \bf +6.14 \\
				C-FSCIL \cite{hersche2022constrained} &77.47	&72.40	&67.47	&63.25	&59.84	&56.95	&54.42	&52.47	&50.47	&61.64& \bf + 5.64\\
				Data-free Replay \cite{liu2022few} &74.40	&70.20	&66.54&	62.51	&59.71	&56.58	&54.52&	52.39	&50.14	&60.78& \bf +5.97\\
				ALICE \cite{peng2022few} &79.00	&70.50	&67.10	&63.40	&61.20	&59.20	&58.10	&56.30	&54.10	&63.21& \bf +2.01\\
				\midrule
				\bf NC-FSCIL (ours) & \bf 82.52	&\bf 76.82	&\bf 73.34	&\bf 69.68	&\bf 66.19	&\bf 62.85	&\bf 60.96	&\bf 59.02	&\bf 56.11	&\bf 67.50 &\\
				\midrule
				\emph{Improvement over ALICE} & {+3.52}&	+6.32&	+6.24&	+6.28&	+4.99&	+3.65&	+2.86&	+2.72&	+2.01&	{\color{purple} \textbf{+4.29}}&\\
				\bottomrule
			\end{tabular}
		}
		\label{table:cifar}
	\end{center}
\end{table*}

\begin{table*}[!ht]
	\renewcommand\arraystretch{1}
	\centering
	\tiny
	\caption{Performance of the generalized case (GCIL) on CIFAR-100 by repeating experiments four times with different sampled setups for 10 and 25 steps, respectively. We re-implement the representative studies in CIL (LUCIR \cite{hou2019learning} and AANet \cite{liu2021adaptive}), LTCIL (LT-CIL \cite{liu2022long}), and FSCIL (ALICE \cite{peng2022few}) in the generalized case. ``Last'' and ``Avg'' refer to the accuracy after the last session and the average incremental accuracy, respectively. The detailed setups are shown in the Appendix.}
	\subfloat[10 steps]{
		\resizebox{\textwidth}{!}{
			\begin{tabular}{l|cc|cc|cc|cc}
				\toprule[0.1em]
				\multirow{2}{*}{Method} & \multicolumn{2}{c|}{exp-1} & \multicolumn{2}{c|}{exp-2}& \multicolumn{2}{c|}{exp-3} & \multicolumn{2}{c}{exp-4}\\
				
				& Last (\%) & Avg (\%) & Last (\%) & Avg (\%) & Last (\%) & Avg (\%)  & Last (\%) & Avg (\%) \\
				
				\midrule
				
				LUCIR \cite{hou2019learning}& 44.0 &  52.7 & 16.0 & 56.4 & 6.7 & 49.2 & 13.6 & 52.1 \\
				
				AANet \cite{liu2021adaptive}& 47.7 &  56.8  & 23.9 & 56.9  & 19.3 & 53.3  & 20.7 & 53.3\\
				
				LT-CIL \cite{liu2022long}& 49.3 &  65.2  & 54.8 & 66.6  & 53.3 & 65.8 & 54.8 & 66.6 \\
				
				ALICE \cite{peng2022few}& 11.3 &  23.3  & 11.3 & 23.6  & 11.8 & 23.6  & 12.0 & 24.4 \\
				
				{\textbf{NC-GCIL (Ours)}} & \textbf{54.3} &   \textbf{67.0} &  \textbf{57.9} &  \textbf{68.1} &  \textbf{58.6} &   \textbf{67.2} &  \textbf{58.4} &  \textbf{68.7} \\
				\bottomrule[0.1em]
			\end{tabular}
		}
	}
	\vfill
	\vspace{2mm}
	\subfloat[25 steps]{
		\resizebox{\textwidth}{!}{
			\begin{tabular}{l|cc|cc|cc|cc}
				\toprule[0.1em]
				\multirow{2}{*}{Method} & \multicolumn{2}{c|}{exp-5} & \multicolumn{2}{c|}{exp-6}& \multicolumn{2}{c|}{exp-7} & \multicolumn{2}{c}{exp-8}\\
				
				& Last (\%) & Avg (\%) & Last (\%) & Avg (\%) & Last (\%) & Avg (\%) & Last (\%) & Avg (\%) \\
				\midrule
				LUCIR   \cite{hou2019learning}& 40.4 &  56.5  & 21.7 & 55.2 & 6.4 & 50.6 & 50.4 & 57.8 \\
				
				AANet \cite{liu2021adaptive}& 45.4 &  56.9  & 25.8 & 51.0 & 44.9 & 53.6  & 49.1 & 57.2 \\
				
				LT-CIL \cite{liu2022long}& 44.8 &  63.1 & 49.4 & 64.5  & 1.0 & 62.0 & 48.4 & 64.3 \\
				
				ALICE \cite{peng2022few}& 11.6 &  20.3 & 12.9 & 20.6  & 12.4 & 21.1  & 12.2 & 21.0 \\
				
				\textbf{NC-GCIL (Ours)} & \textbf{54.6} &  \textbf{66.4}  & \textbf{54.4} & \textbf{66.3}  & \textbf{50.1} & \textbf{65.6}  & \textbf{56.3} & \textbf{67.6} \\
				
				\bottomrule[0.1em]
			\end{tabular}
		}
	}
	\label{tab:generalized}
\end{table*}

\noindent
\textbf{Results on FSCIL.} Our experiment results on minImageNet and CIFAR-100 are shown in Table~\ref{table:imgnet} and Table~\ref{table:cifar}, respectively. The result on CUB-200 is shown in the Appendix.
We see that our method achieves the best performance in all sessions on both miniImageNet and CIFAR-100 compared with previous studies. ALICE \cite{peng2022few} is a recent study that achieves strong performances on FSCIL. Compared with this challenging baseline, we have an improvement of 2.61\% in the last session on miniImageNet, and 2.01\% on CIFAR-100. We achieve an average accuracy improvement of more than 3.5\% on both miniImageNet and CIFAR-100. Although we do not surpass ALICE in the last session on CUB-200, we still have the best average incremental accuracy among all methods. As shown in the last rows of Table \ref{table:imgnet} and Table \ref{table:cifar}, the improvement of our method is able to last and even becomes larger in the first several sessions. It indicates that our method has the ability to hold the superiority and relieve the forgetting on old  knowledge.

\begin{table*}[th]
     \caption{Ablation studies on ImageNet-100 for CIL. We compare the {average incremental accuracy} in different settings. Please refer to each paragraph in Section~\ref{sec-ablation-cil} for detailed descriptions and analyses. The lines with gray background indicate our default settings.}
     \subfloat[ablation study on $\lambda$.]{
        \begin{tabular}{c|ccc}
        \toprule[0.15em]
        $\lambda$ & \footnotesize{5 steps} & \footnotesize{10 steps} & \footnotesize{25 steps}\\
        \midrule
        1 & 77.92 & 74.42 & 72.63\\
        2 & 79.07 & 76.33 & 74.93\\
        \rowcolor{gray!15} 5 & 79.29 & 77.62 & 74.97\\
        10 & 77.92 & 77.32 & 74.94\\
        \bottomrule[0.1em]
        \end{tabular}
        \label{tab:ablation_a}
     }
     \hfill
     \subfloat[ablation study on loss functions.]{
     	\setlength{\tabcolsep}{5pt}
        \begin{tabular}{c|ccc}
        \toprule[0.15em]
        $\mathcal{L}$ & \footnotesize{5 steps} & \footnotesize{10 steps} & \footnotesize{25 steps}\\
        \midrule
        {Learnable}~\cite{hou2019learning} & 66.27 & 60.80 & 52.96\\
        \footnotesize{$\gL_{\rm{CE}}$+$\gL_{\rm{Dot}}$} & 77.27 & 73.50 & 72.06\\    
        \footnotesize{$\gL_{\rm{CE}}$+$\gL_{\rm{distill}}$} & 77.78 & 77.29 & 74.95\\    
        \rowcolor{gray!15} \footnotesize{$\gL_{\rm{align}}$+$\gL_{\rm{distill}}$} & 79.29 & 77.62 & 74.97\\  
        \bottomrule[0.1em]
        \end{tabular}
        \label{tab:ablation_b}
     }
     \hfill
     \subfloat[ablation study on classifier choice.]{
     	     	\setlength{\tabcolsep}{4pt}
        \begin{tabular}{c|ccc}
        \toprule[0.15em]
        Method & \footnotesize{5 steps} & \footnotesize{10 steps} & \footnotesize{25 steps}\\
        \midrule
        {Learnable~\cite{hou2019learning}} & 66.27 & 60.80 & 52.96\\
        \footnotesize{NCM only} & 62.72 & 56.22 & 56.99\\
        \footnotesize{NCT only} & 78.16 & 75.98 & 73.59\\
        \rowcolor{gray!15} \footnotesize{Flying to Collapse} & 79.29 & 77.62 & 74.97\\ 
        \bottomrule[0.1em]
        \end{tabular}
        \label{tab:ablation_c}
     }\hfill
     \vspace{2mm}

     \hspace{12mm}\subfloat[orthonogal frame (OF) v.s. simplex ETF.]{
    	\begin{tabular}{l|ccc}
		\toprule[0.15em]
		Method & 5 steps & 10 steps & 25 steps\\
		\midrule
		OF only & 76.23 & 74.32 & 72.19\\
		FtC with OF & 78.03 & 76.05 & 73.31\\
		\rowcolor{gray!15} FtC with NCT  & 79.29 & 77.62 & 74.97 \\
		\bottomrule[0.1em]
     	\end{tabular}
     	\label{tab:ablation_ortho}
     }\hfill  
      \subfloat[class prototype number.]{
    	\begin{tabular}{l|ccc}
		\toprule[0.15em]
		Method & 5 steps & 10 steps & 25 steps\\
		\midrule
		\rowcolor{gray!15} {NCT with 100 prototypes} & 79.29 & 77.62 & 74.97\\
		{NCT with 500 prototypes} &78.85 & 77.58 & 73.54\\
		{NCT with 1000 prototypes} & 78.32 & 77.01 & 74.66 \\
		\bottomrule[0.1em]
		\end{tabular}
     	\label{tab:ablation_classnum}
     }
     \hspace{10mm}
     \label{tab:ablation}
\end{table*}

\begin{table*}[t]
	\begin{center}
		\caption{Ablation studies on three datasets for FSCIL. ``Learnable+$\gL_{\rm{CE}}$'' uses a learnable classifier and the CE loss; ``NCT+$\gL_{\rm{CE}}$'' adopts our neural collapse terminus with the CE loss; ``NCT+$\gL_{\rm{align}}$'' uses both neural collapse terminus and the misalignment loss. ``{Final}'' refers to the accuracy after the last session; ``{Average}'' is the average incremental accuracy; ``PD'' denotes the performance drop, \emph{i.e.,} the accuracy difference between the first and the last sessions.}
		\setlength{\tabcolsep}{6pt}
		{
			\begin{tabular}{l|ccc|ccc|ccc}
				\toprule[0.15em]
				\multicolumn{1}{l}{\multirow{2}{*}{Methods}} & \multicolumn{3}{c}{miniImageNet} & \multicolumn{3}{c}{CIFAR-100} & \multicolumn{3}{c}{CUB-200}\\
				& {Final}$\uparrow$ & {Average}$\uparrow$ & {PD}$\downarrow$ & {Final}$\uparrow$ & {Average}$\uparrow$ & {PD}$\downarrow$ & {Final}$\uparrow$ & {Average}$\uparrow$ & {PD}$\downarrow$\\
				\midrule 
				Learnable+$\gL_{\rm{CE}}$  & 50.04 & 61.30 & 34.53 & 52.13 & 62.68 & 30.14 & 50.38 & 59.58 & 29.19 \\
				NCT+$\gL_{\rm{CE}}$ & 56.66 & 68.23 & 28.21 & 54.42 & 64.00 & 27.36 & 56.83 & 65.51 & 23.27 \\
				\rowcolor{gray!15} NCT+$\gL_{\rm{align}}$ (ours) & 58.31 & 67.82 &  25.71 & 56.11 & 67.50 & 26.41 & 59.44 & 67.28 & 21.01\\
				\bottomrule
			\end{tabular}
		}
		\label{ablation}
	\end{center}
\end{table*}

\noindent
\textbf{Results on GCIL.} 
We compare our method with the representative methods in CIL, LTCIL, and FSCIL in the generalized case (GCIL). Considering data distribution in each new session could be any of the three cases by sampling, we repeat experiment multiple times. The detailed sampled setups are recorded in the Appendix. As shown in Table \ref{tab:generalized}, our method achieves the best last and average accuracies in all the 8 experiments. Because the compared methods are specifically designed for CIL, LTCIL, or FSCIL, they fail (less than 20\%) in the last sessions of some experiments, such as LUCIR \cite{hou2019learning} in exp-3, AANet \cite{liu2021adaptive} in exp-3, and LT-CIL \cite{liu2022long} in exp-7. ALICE \cite{peng2022few} fails in the all experiments, which may be due to their fixing the model for incremental sessions. LT-CIL \cite{liu2022long} adopts a two-stage method so consumes more train time than the other methods in all the experiments, as shown in the Appendix. 
By contrast, our NC-GCIL, as the first one-stage solution to LTCIL, adopts a simple framework of LUCIR, but still achieves the best results in all the experiments, which further corroborates the generalizability of our unified solution to all the three class incremental learning problems. Moreover, in these experiments, we adopt a neural collapse terminus of 1,000 prototypes but only use 100 of them to train on CIFAR-100, which verifies the feasibility of our method in the case when we do not know the total class number of a problem.

\subsection{Ablation Studies}
\label{sec-ablation}

We conduct ablation studies for both CIL and FSCIL
to delve into the effectiveness of our proposed method.

\subsubsection{Ablation for CIL}
\label{sec-ablation-cil}

\noindent
\textbf{Hyper-parameter $\lambda$.}
In Eq.~(\ref{loss_final}), our final loss incorporates $\lambda$ to balance between $\gL_{\rm{align}}$ and $\gL_{\rm{distill}}$. We conduct an ablation on the effect of $\lambda$, as shown in Table~\ref{tab:ablation_a}.
Similar to LUCIR~\cite{hou2019learning}, we find that $\lambda = 5$ also works well in our framework ($\lambda$ here is actually the $\lambda_{base}$ in LUCIR~\cite{hou2019learning}). 
The experiment in Table~\ref{tab:ablation_a} indicates that the whole framework is not sensitive to the choice of $\lambda$. From $\lambda=2$ to $\lambda=10$, our method has stable performances and outperforms most of the state-of-the-art results in Table~\ref{tab:com}.

\noindent
\textbf{Loss Choice.} As mentioned in Section~\ref{cil-training}, we adopt the misalignment loss~\cite{yang2022we} for both aligning backbone features by Eq. (\ref{dr}) and distillation by Eq. (\ref{dr_dist}).
We test the different choices of loss function for classification and distillation. 
In LUCIR~\cite{hou2019learning}, they use Cross-Entropy Loss ($\gL_{\rm{CE}}$) for classification and DotLoss ($\gL_{\rm{Dot}}$) for distillation. 
Accordingly, we test the models including: 
1) the same loss choice as LUCIR ($\gL_{\rm{CE}}$+$\gL_{\rm{Dot}}$) implemented by our method; 
2) CELoss for classification, our misalignment loss Eq. (\ref{dr_dist}) for distillation ($\gL_{\rm{CE}}$+$\gL_{\rm{distill}}$); 
3) misalignment loss for both classification and distillation ($\gL_{\rm{align}}$+$\gL_{\rm{distill}}$), which is our default choice.  
As shown in Table~\ref{tab:ablation_b}, using our adopted misalignment loss for both classification and distillation achieves the best result. All the three choices achieve significantly better performances than the LUCIR baseline whose classifier is learnable (the first row), which indicates that the main performance gain of our method is attributed to the proposed neural collapse terminus and flying to collapse. 

\noindent
\textbf{Classifier Choice.} We study different classifier strategies in Table~\ref{tab:ablation_c}-\ref{tab:ablation_classnum}. Specifically, in Table \ref{tab:ablation_c}, we consider: 1) NCM only, \emph{i.e.,} $\eta=0$ in Eq.~(\ref{fly}); 2) NCT only without flying to collapse (FtC), \emph{i.e.,} $\eta=1$ in Eq.~(\ref{fly}); 3) FtC (gradually evolving from NCM to NCT as our method introduced in Section~\ref{cil-fly}).
It is shown that using ``NCT only'' largely improves the performance over the learnable classifier in~\cite{hou2019learning}. When evolving the classifier from NCM to NCT to drive the backbone features smoothly, we further make a better result, demonstrating the importance of FtC. In Table~\ref{tab:ablation_ortho}, we compare the classifier geometric structure between an orthogonal frame and a simplex equiangular tight frame. It is shown that ''NCT only`` (without FtC) is better than ``OF only'', and ``FtC with NCT'' is better than ``FtC with OF'', which means a larger margin for the preassigned feature space partition helps class incremental learning. In Table \ref{tab:ablation_classnum}, we ablate the number of class prototypes. When we do not know the total class number for a class incremental learning problem, as introduced in Section~\ref{generalized_case}, we can initialize an NCT with a large number of prototypes such that the classes to be encountered in reality will not surpass what we have preassigned. It is shown that when our NCT has much more prototypes, the performance only diminishes within a small scope, which validates the feasibility of this strategy. 

\begin{figure*}[!ht]
	\centering
	\begin{subfigure}[b]{1.\linewidth}
		\centering
		\includegraphics[width=0.975\linewidth]{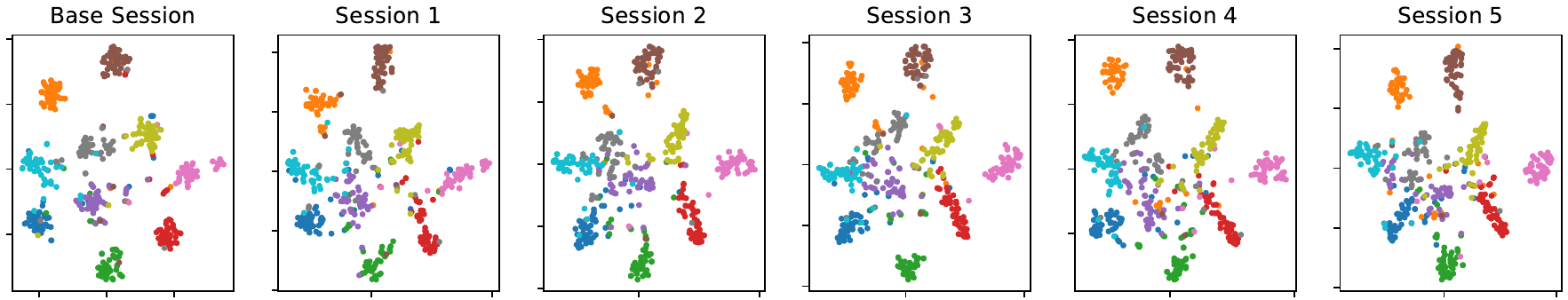}
	\end{subfigure}
	\begin{subfigure}[b]{1.\linewidth}
		\centering
		\includegraphics[width=0.975\linewidth]{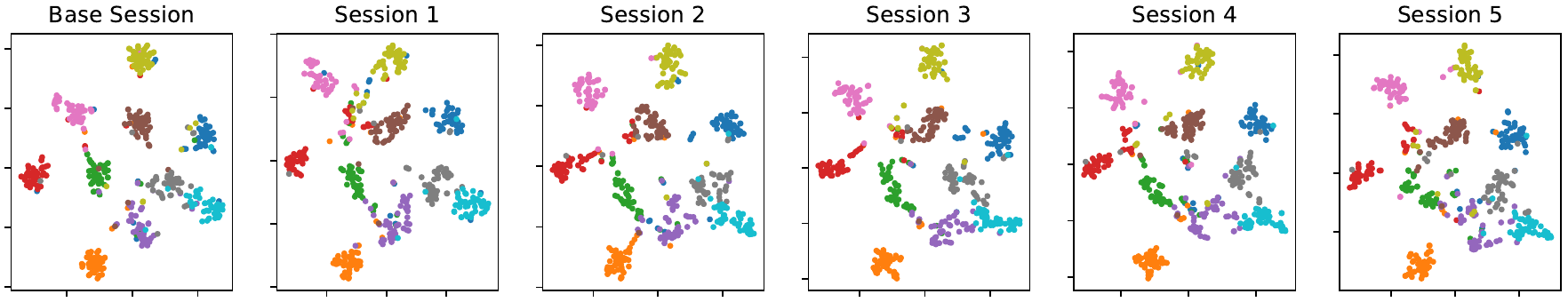}
	\end{subfigure}
	\caption{The t-SNE~\cite{van2008visualizing} visualization results for the features of the base session classes using ImageNet-100 validation set in ``CIL with 5 steps''. {Up}: LUCIR~\cite{hou2019learning}. {Down}: our NC-CIL. We use the same t-SNE settings for fair comparison.}
	\label{fig:tsne}
	\vspace{-1mm}
\end{figure*}

\begin{figure*}[!ht]
	\begin{subfigure}{.245\textwidth}
		\includegraphics[width=0.95\linewidth]{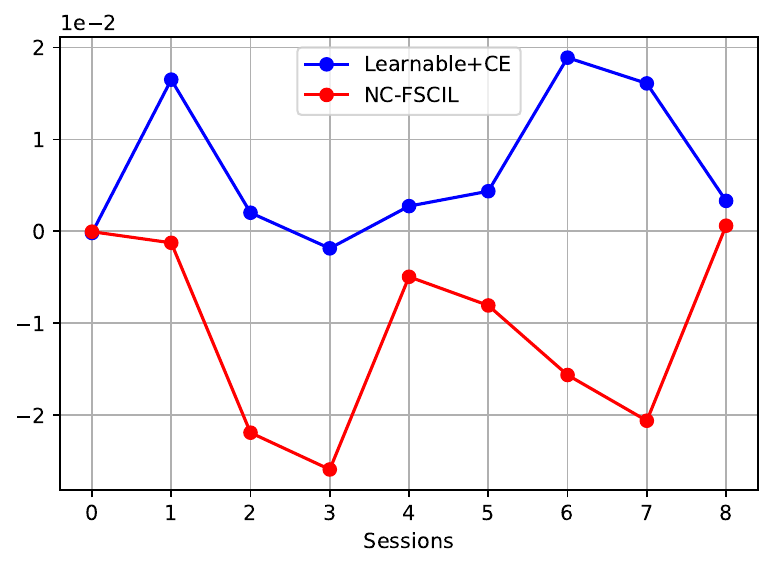}  
		\caption{train (each)}
		\label{fig:sub-first}
	\end{subfigure}
	\begin{subfigure}{.245\textwidth}
		\includegraphics[width=0.95\linewidth]{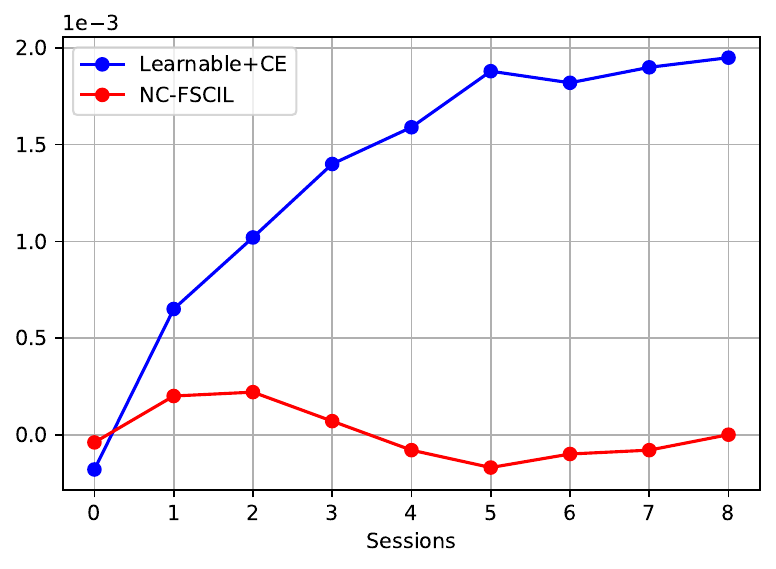}  
		\caption{train (accumulate)}
		\label{fig:sub-second}
	\end{subfigure}
	\begin{subfigure}{.245\textwidth}
		\includegraphics[width=0.95\linewidth]{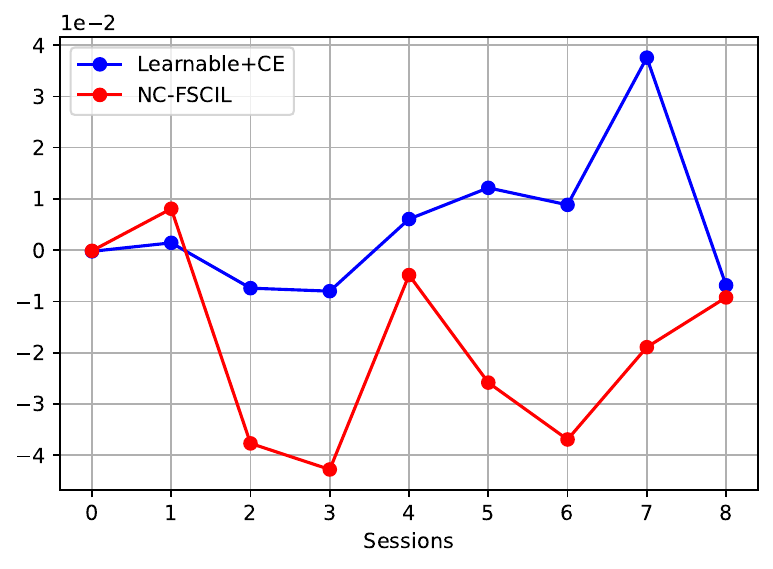}  
		\caption{test (each)}
		\label{fig:sub-third}
	\end{subfigure}
	\begin{subfigure}{.245\textwidth}
		\includegraphics[width=0.95\linewidth]{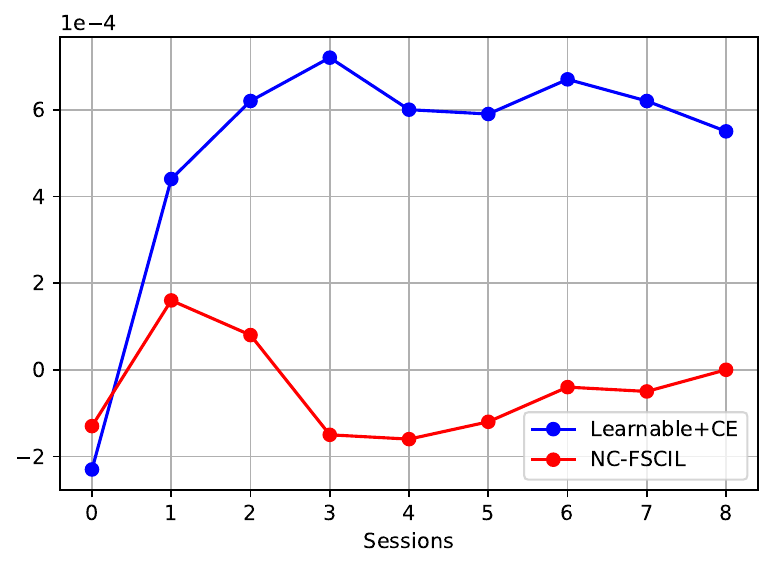} 
		\caption{test (accumulate)}
		\label{fig:sub-fourth}
	\end{subfigure}
	\caption{Average cosine similarity between features and classifier prototypes of different classes, \emph{i.e.,} $\mathrm{Avg}_{k\ne k'}\{\cos\angle(\rvm_{k}-\rvm_{G}, {\rvw}_{k'}) \}$, where $\rvm_{k}$ is the within-class feature mean of class $k$, $\rvm_{G}$ denotes the global mean, and ${\rvw}_{k'}$ is the classifier prototype of class $k'$. Statistics are performed among classes of each session (a and c), and all the encountered classes by the current session (b and d), on train set (a and b) and test set (c and d), for our NC-FSCIL on miniImageNet.}
	\label{fig:fig}
\end{figure*}

\subsubsection{Ablation for FSCIL}
We consider three models to validate the effects of NCT and misalignment loss. All three models are based on the same framework introduced in Section~\ref{method_FSCIL}, including the backbone network, the projection layer, and the memory module. The first model uses a learnable classifier and the CE loss, which is the most adopted practice. The second model only replaces the classifier with our NCT and still uses the CE loss. The third model corresponds to our method using both NCT and the misalignment loss $\gL_{\rm{align}}$ in Eq. (\ref{dr}). As shown in Table~\ref{ablation}, when our NCT is used, the final session accuracies are significantly better, and the performance drops get much mitigated. Adopting the misalignment loss $\gL_{\rm{align}}$ 
is able to further moderately improve the performances. {It indicates that the success of our method for FSCIL is largely attributed to the NCT and misalignment loss, as they pre-assign a neural collapse inspired consistent target and drive backbone features towards this optimality throughout incremental training, respectively}. 

\subsection{Analysis}

We visualize the features of the base session classes after each session's training in CIL. As shown in Figure~\ref{fig:tsne}, in our baseline LUCIR, as incremental training goes on, the separated clusters are getting interfaced. In the last session, it is hard to differentiate among the classes by their locations on the plane. By contrast, there is still a clear margin between any two clusters in the last session of our method. It is in line with our ability in mitigating the forgetting on old classes. 


We also investigate the feature-classifier status in our NC-FSCIL and the one using a learnable classifier and the CE loss.
As shown in Figure \ref{fig:fig}, the average cosine similarities between features and classifier prototypes of different classes, \emph{i.e.,} $\mathrm{Avg}_{k\ne k'}\{\cos\angle(\rvm_{k}-\rvm_{G}, {\rvw}_{k'}) \}$, of our method are consistently lower than those of the baseline. Most values of our method are negative and close to 0, which corresponds to the guidance from neural collapse as derived in Eq. (\ref{theorem_eq}). 
Particularly in Figures~\ref{fig:sub-second} and \ref{fig:sub-fourth}, the cosine similarities averaged
among all encountered classes increase fast with session for the baseline method, while ours keeps relatively flat. It indicates that the baseline method reduces the feature-classifier margin of different classes as training incrementally, but our method enjoys a stabler between-class separation. 
In the Appendix, 
we also calculate the average cosine similarities between features and classifier prototypes of the same class, \emph{i.e.,} $\mathrm{Avg}_{k}\{\cos\angle(\rvm_{k}-\rvm_{G}, {\rvw}_{k}) \}$, and the trace ratio of within-class covariance to between-class covariance, ${{\rm tr}(\Sigma_W)}/{{\rm tr}(\Sigma_B)}$. These results together support that our method better holds the feature-classifier alignment and relieves the forgetting problem.

\section{Conclusion}

In this paper, we point out a shared obstacle, misalignment dilemma, for the three class incremental learning problems including the cases of normal, long-tail, and few-shot data distribution. Inspired by neural collapse, we propose a unified solution that pre-assigns an optimal feature-classifier status as instructed by neural collapse for the whole label space, such that all sessions in incremental training have a consistent target to avoid dividing feature space incrementally. We develop a prototype evolving strategy and adopt a novel misalignment loss to further mitigate the catastrophic forgetting problem. Our method works for both CIL and LTCIL without any change on the architecture or training setting, and only needs minor adaptations for FSCIL. We further develop a generalized case that disables the prior knowledge of data distribution and the total class number. In experiments, our method improves significantly over the state-of-the-art methods on all the three tasks and the more challenging generalized case.

\ifCLASSOPTIONcaptionsoff
  \newpage
\fi



\bibliographystyle{IEEEtran}
\bibliography{UniCIL-TPAMI}

\begin{thebibliography}{100}
\providecommand{\url}[1]{#1}
\csname url@samestyle\endcsname
\providecommand{\newblock}{\relax}
\providecommand{\bibinfo}[2]{#2}
\providecommand{\BIBentrySTDinterwordspacing}{\spaceskip=0pt\relax}
\providecommand{\BIBentryALTinterwordstretchfactor}{4}
\providecommand{\BIBentryALTinterwordspacing}{\spaceskip=\fontdimen2\font plus
\BIBentryALTinterwordstretchfactor\fontdimen3\font minus
  \fontdimen4\font\relax}
\providecommand{\BIBforeignlanguage}[2]{{%
\expandafter\ifx\csname l@#1\endcsname\relax
\typeout{** WARNING: IEEEtran.bst: No hyphenation pattern has been}%
\typeout{** loaded for the language `#1'. Using the pattern for}%
\typeout{** the default language instead.}%
\else
\language=\csname l@#1\endcsname
\fi
#2}}
\providecommand{\BIBdecl}{\relax}
\BIBdecl

\bibitem{lecun2015deep}
Y.~LeCun, Y.~Bengio, and G.~Hinton, ``Deep learning,'' \emph{Nature}, vol. 521,
  no. 7553, pp. 436--444, 2015.

\bibitem{cauwenberghs2000incremental}
G.~Cauwenberghs and T.~Poggio, ``Incremental and decremental support vector
  machine learning,'' in \emph{NeurIPS}, vol.~13, 2000.

\bibitem{polikar2001learn++}
R.~Polikar, L.~Upda, S.~S. Upda, and V.~Honavar, ``Learn++: An incremental
  learning algorithm for supervised neural networks,'' \emph{IEEE Transactions
  on Systems, Man, and Cybernetics, part C (applications and reviews)},
  vol.~31, no.~4, pp. 497--508, 2001.

\bibitem{mensink2013distance}
T.~Mensink, J.~Verbeek, F.~Perronnin, and G.~Csurka, ``Distance-based image
  classification: Generalizing to new classes at near-zero cost,''
  \emph{TPAMI}, vol.~35, no.~11, pp. 2624--2637, 2013.

\bibitem{li2017learning}
Z.~Li and D.~Hoiem, ``Learning without forgetting,'' \emph{TPAMI}, vol.~40,
  no.~12, pp. 2935--2947, 2017.

\bibitem{rebuffi2017icarl}
S.-A. Rebuffi, A.~Kolesnikov, G.~Sperl, and C.~H. Lampert, ``icarl: Incremental
  classifier and representation learning,'' in \emph{CVPR}, 2017.

\bibitem{liu2022long}
X.~Liu, Y.-S. Hu, X.-S. Cao, A.~D. Bagdanov, K.~Li, and M.-M. Cheng,
  ``Long-tailed class incremental learning,'' in \emph{ECCV}, 2022.

\bibitem{tao2020few}
X.~Tao, X.~Hong, X.~Chang, S.~Dong, X.~Wei, and Y.~Gong, ``Few-shot
  class-incremental learning,'' in \emph{CVPR}, 2020.

\bibitem{goodfellow2013empirical}
I.~J. Goodfellow, M.~Mirza, D.~Xiao, A.~Courville, and Y.~Bengio, ``An
  empirical investigation of catastrophic forgetting in gradient-based neural
  networks,'' \emph{arXiv preprint arXiv:1312.6211}, 2013.

\bibitem{liu2020mnemonics}
Y.~Liu, Y.~Su, A.-A. Liu, B.~Schiele, and Q.~Sun, ``Mnemonics training:
  Multi-class incremental learning without forgetting,'' in \emph{CVPR}, 2020.

\bibitem{zhao2020maintaining}
B.~Zhao, X.~Xiao, G.~Gan, B.~Zhang, and S.-T. Xia, ``Maintaining discrimination
  and fairness in class incremental learning,'' in \emph{CVPR}, 2020.

\bibitem{iscen2020memory}
A.~Iscen, J.~Zhang, S.~Lazebnik, and C.~Schmid, ``Memory-efficient incremental
  learning through feature adaptation,'' in \emph{ECCV}, 2020.

\bibitem{yu2020semantic}
L.~Yu, B.~Twardowski, X.~Liu, L.~Herranz, K.~Wang, Y.~Cheng, S.~Jui, and
  J.~v.~d. Weijer, ``Semantic drift compensation for class-incremental
  learning,'' in \emph{CVPR}, 2020.

\bibitem{hu2021distilling}
X.~Hu, K.~Tang, C.~Miao, X.-S. Hua, and H.~Zhang, ``Distilling causal effect of
  data in class-incremental learning,'' in \emph{CVPR}, 2021.

\bibitem{yan2021dynamically}
S.~Yan, J.~Xie, and X.~He, ``Der: Dynamically expandable representation for
  class incremental learning,'' in \emph{CVPR}, 2021.

\bibitem{liu2021adaptive}
Y.~Liu, B.~Schiele, and Q.~Sun, ``Adaptive aggregation networks for
  class-incremental learning,'' in \emph{CVPR}, 2021.

\bibitem{wu2021striking}
G.~Wu, S.~Gong, and P.~Li, ``Striking a balance between stability and
  plasticity for class-incremental learning,'' in \emph{ICCV}, 2021.

\bibitem{pernici2021class}
F.~Pernici, M.~Bruni, C.~Baecchi, F.~Turchini, and A.~Del~Bimbo,
  ``Class-incremental learning with pre-allocated fixed classifiers,'' in
  \emph{ICPR}, 2021, pp. 6259--6266.

\bibitem{zhou2021co}
D.-W. Zhou, H.-J. Ye, and D.-C. Zhan, ``Co-transport for class-incremental
  learning,'' in \emph{ACM MM}, 2021.

\bibitem{liu2021rmm}
Y.~Liu, B.~Schiele, and Q.~Sun, ``Rmm: Reinforced memory management for
  class-incremental learning,'' in \emph{NeurIPS}, 2021.

\bibitem{ni2021alleviate}
Z.~Ni, H.~Shi, S.~Tang, L.~Wei, Q.~Tian, and Y.~Zhuang, ``Revisiting
  catastrophic forgetting in class incremental learning,'' \emph{arXiv
  preprint}, 2021.

\bibitem{liu2022model}
Y.~Liu, X.~Hong, X.~Tao, S.~Dong, J.~Shi, and Y.~Gong, ``Model behavior
  preserving for class-incremental learning,'' \emph{TNNLS}, 2022.

\bibitem{shi2022mimicking}
Y.~Shi, K.~Zhou, J.~Liang, Z.~Jiang, J.~Feng, P.~H. Torr, S.~Bai, and V.~Y.
  Tan, ``Mimicking the oracle: An initial phase decorrelation approach for
  class incremental learning,'' in \emph{CVPR}, 2022.

\bibitem{wang2022foster}
F.-Y. Wang, D.-W. Zhou, H.-J. Ye, and D.-C. Zhan, ``Foster: Feature boosting
  and compression for class-incremental learning,'' in \emph{ECCV}, 2022.

\bibitem{pourkeshavarzi2022looking}
M.~PourKeshavarzi, G.~Zhao, and M.~Sabokrou, ``Looking back on learned
  experiences for class/task incremental learning,'' in \emph{ICLR}, 2022.

\bibitem{zhu2021class}
F.~Zhu, Z.~Cheng, X.-y. Zhang, and C.-l. Liu, ``Class-incremental learning via
  dual augmentation,'' in \emph{NeurIPS}, 2021.

\bibitem{bhunia2022doodle}
A.~K. Bhunia, V.~R. Gajjala, S.~Koley, R.~Kundu, A.~Sain, T.~Xiang, and Y.-Z.
  Song, ``Doodle it yourself: Class incremental learning by drawing a few
  sketches,'' in \emph{CVPR}, 2022.

\bibitem{ashok2022class}
A.~Ashok, K.~Joseph, and V.~N. Balasubramanian, ``Class-incremental learning
  with cross-space clustering and controlled transfer,'' in \emph{ECCV}, 2022.

\bibitem{simon2021learning}
C.~Simon, P.~Koniusz, and M.~Harandi, ``On learning the geodesic path for
  incremental learning,'' in \emph{CVPR}, 2021.

\bibitem{hou2019learning}
S.~Hou, X.~Pan, C.~C. Loy, Z.~Wang, and D.~Lin, ``Learning a unified classifier
  incrementally via rebalancing,'' in \emph{CVPR}, 2019.

\bibitem{douillard2020podnet}
A.~Douillard, M.~Cord, C.~Ollion, T.~Robert, and E.~Valle, ``Podnet: Pooled
  outputs distillation for small-tasks incremental learning,'' in \emph{ECCV},
  2020.

\bibitem{dong2021few}
S.~Dong, X.~Hong, X.~Tao, X.~Chang, X.~Wei, and Y.~Gong, ``Few-shot
  class-incremental learning via relation knowledge distillation,'' in
  \emph{AAAI}, vol.~35, no.~2, 2021, pp. 1255--1263.

\bibitem{vinyals2016matching}
O.~Vinyals, C.~Blundell, T.~Lillicrap, D.~Wierstra \emph{et~al.}, ``Matching
  networks for one shot learning,'' in \emph{NeurIPS}, vol.~29, 2016.

\bibitem{ravi2017optimization}
S.~Ravi and H.~Larochelle, ``Optimization as a model for few-shot learning,''
  in \emph{ICLR}, 2017.

\bibitem{snell2017prototypical}
J.~Snell, K.~Swersky, and R.~Zemel, ``Prototypical networks for few-shot
  learning,'' in \emph{NeurIPS}, vol.~30, 2017.

\bibitem{sung2018learning}
F.~Sung, Y.~Yang, L.~Zhang, T.~Xiang, P.~H. Torr, and T.~M. Hospedales,
  ``Learning to compare: Relation network for few-shot learning,'' in
  \emph{CVPR}, 2018, pp. 1199--1208.

\bibitem{zhang2021few}
C.~Zhang, N.~Song, G.~Lin, Y.~Zheng, P.~Pan, and Y.~Xu, ``Few-shot incremental
  learning with continually evolved classifiers,'' in \emph{CVPR}, 2021.

\bibitem{hersche2022constrained}
M.~Hersche, G.~Karunaratne, G.~Cherubini, L.~Benini, A.~Sebastian, and
  A.~Rahimi, ``Constrained few-shot class-incremental learning,'' in
  \emph{CVPR}, 2022.

\bibitem{akyurek2021subspace}
A.~F. Aky{\"u}rek, E.~Aky{\"u}rek, D.~Wijaya, and J.~Andreas, ``Subspace
  regularizers for few-shot class incremental learning,'' in \emph{ICLR}, 2022.

\bibitem{chen2021incremental}
K.~Chen and C.-G. Lee, ``Incremental few-shot learning via vector quantization
  in deep embedded space,'' in \emph{ICLR}, 2021.

\bibitem{zhou2022forward}
D.-W. Zhou, F.-Y. Wang, H.-J. Ye, L.~Ma, S.~Pu, and D.-C. Zhan, ``Forward
  compatible few-shot class-incremental learning,'' in \emph{CVPR}, 2022.

\bibitem{biondi2023cores}
N.~Biondi, F.~Pernici, M.~Bruni, and A.~Del~Bimbo, ``Cores: Compatible
  representations via stationarity,'' \emph{TPAMI}, 2023.

\bibitem{papyan2020prevalence}
V.~Papyan, X.~Han, and D.~L. Donoho, ``Prevalence of neural collapse during the
  terminal phase of deep learning training,'' \emph{PNAS}, 2020.

\bibitem{martinez2001pca}
A.~M. Martinez and A.~C. Kak, ``Pca versus lda,'' \emph{TPAMI}, vol.~23, no.~2,
  pp. 228--233, 2001.

\bibitem{fisher1936use}
R.~A. Fisher, ``The use of multiple measurements in taxonomic problems,''
  \emph{Annals of eugenics}, vol.~7, no.~2, pp. 179--188, 1936.

\bibitem{rao1948utilization}
C.~R. Rao, ``The utilization of multiple measurements in problems of biological
  classification,'' \emph{Journal of the Royal Statistical Society. Series B
  (Methodological)}, vol.~10, no.~2, pp. 159--203, 1948.

\bibitem{fang2021exploring}
C.~Fang, H.~He, Q.~Long, and W.~J. Su, ``Exploring deep neural networks via
  layer-peeled model: Minority collapse in imbalanced training,'' \emph{PNAS},
  vol. 118, no.~43, 2021.

\bibitem{han2022neural}
X.~Han, V.~Papyan, and D.~L. Donoho, ``Neural collapse under {MSE} loss:
  Proximity to and dynamics on the central path,'' in \emph{ICLR}, 2022.

\bibitem{graf2021dissecting}
F.~Graf, C.~Hofer, M.~Niethammer, and R.~Kwitt, ``Dissecting supervised
  constrastive learning,'' in \emph{ICML}, 2021.

\bibitem{yang2022we}
Y.~Yang, L.~Xie, S.~Chen, X.~Li, Z.~Lin, and D.~Tao, ``Inducing neural collapse
  in imbalanced learning: Do we really need a learnable classifier at the end
  of deep neural network?'' in \emph{NeurIPS}, 2022.

\bibitem{xie2022neural}
L.~Xie, Y.~Yang, D.~Cai, and X.~He, ``Neural collapse inspired
  attraction-repulsion-balanced loss for imbalanced learning,''
  \emph{Neurocomputing}, 2023.

\bibitem{yang2023neural}
Y.~Yang, H.~Yuan, X.~Li, Z.~Lin, P.~Torr, and D.~Tao, ``Neural collapse
  inspired feature-classifier alignment for few-shot class-incremental
  learning,'' in \emph{ICLR}, 2023.

\bibitem{aljundi2018memory}
R.~Aljundi, F.~Babiloni, M.~Elhoseiny, M.~Rohrbach, and T.~Tuytelaars, ``Memory
  aware synapses: Learning what (not) to forget,'' in \emph{ECCV}, 2018.

\bibitem{lopez2017gradient}
D.~Lopez-Paz and M.~Ranzato, ``Gradient episodic memory for continual
  learning,'' in \emph{NIPS}, 2017.

\bibitem{mirzadeh2020understanding}
S.~I. Mirzadeh, M.~Farajtabar, R.~Pascanu, and H.~Ghasemzadeh, ``Understanding
  the role of training regimes in continual learning,'' in \emph{NeurIPS},
  2020.

\bibitem{van2022three}
G.~M. van~de Ven, T.~Tuytelaars, and A.~S. Tolias, ``Three types of incremental
  learning,'' \emph{Nature Machine Intelligence}, vol.~4, no.~12, pp.
  1185--1197, 2022.

\bibitem{hinton2014distilling}
G.~Hinton, O.~Vinyals, and J.~Dean, ``Distilling the knowledge in a neural
  network,'' in \emph{NIPS Workshops}, 2014.

\bibitem{li2016learning}
Z.~Li and D.~Hoiem, ``Learning without forgetting,'' in \emph{ECCV}, 2016.

\bibitem{zhou2019m2kd}
P.~Zhou, L.~Mai, J.~Zhang, N.~Xu, Z.~Wu, and L.~S. Davis, ``M2kd: Multi-model
  and multi-level knowledge distillation for incremental learning,''
  \emph{arXiv preprint}, 2019.

\bibitem{castro2018end}
F.~M. Castro, M.~J. Mar{\'\i}n-Jim{\'e}nez, N.~Guil, C.~Schmid, and K.~Alahari,
  ``End-to-end incremental learning,'' in \emph{ECCV}, 2018.

\bibitem{wu2019large}
Y.~Wu, Y.~Chen, L.~Wang, Y.~Ye, Z.~Liu, Y.~Guo, and Y.~Fu, ``Large scale
  incremental learning,'' in \emph{CVPR}, 2019.

\bibitem{shin2017continual}
H.~Shin, J.~K. Lee, J.~Kim, and J.~Kim, ``Continual learning with deep
  generative replay,'' in \emph{NIPS}, 2017.

\bibitem{williams1992simple}
R.~J. Williams, ``Simple statistical gradient-following algorithms for
  connectionist reinforcement learning,'' \emph{Machine learning}, 1992.

\bibitem{douillard2022dytox}
A.~Douillard, A.~Ram{\'e}, G.~Couairon, and M.~Cord, ``Dytox: Transformers for
  continual learning with dynamic token expansion,'' in \emph{CVPR}, 2022.

\bibitem{li2021preserving}
Z.~Li, C.~Zhong, S.~Liu, R.~Wang, and W.-S. Zheng, ``Preserving earlier
  knowledge in continual learning with the help of all previous feature
  extractors,'' \emph{arXiv preprint}, 2021.

\bibitem{abati2020conditional}
D.~Abati, J.~Tomczak, T.~Blankevoort, S.~Calderara, R.~Cucchiara, and B.~E.
  Bejnordi, ``Conditional channel gated networks for task-aware continual
  learning,'' in \emph{CVPR}, 2020.

\bibitem{tao2020topology}
X.~Tao, X.~Chang, X.~Hong, X.~Wei, and Y.~Gong, ``Topology-preserving
  class-incremental learning,'' in \emph{ECCV}, 2020.

\bibitem{dong2022federated}
J.~Dong, L.~Wang, Z.~Fang, G.~Sun, S.~Xu, X.~Wang, and Q.~Zhu, ``Federated
  class-incremental learning,'' in \emph{CVPR}, 2022.

\bibitem{cao2019learning}
K.~Cao, C.~Wei, A.~Gaidon, N.~Arechiga, and T.~Ma, ``Learning imbalanced
  datasets with label-distribution-aware margin loss,'' in \emph{NeurIPS},
  2019.

\bibitem{chu2020feature}
P.~Chu, X.~Bian, S.~Liu, and H.~Ling, ``Feature space augmentation for
  long-tailed data,'' in \emph{ECCV}, 2020.

\bibitem{kang2019decoupling}
B.~Kang, S.~Xie, M.~Rohrbach, Z.~Yan, A.~Gordo, J.~Feng, and Y.~Kalantidis,
  ``Decoupling representation and classifier for long-tailed recognition,'' in
  \emph{ICLR}, 2020.

\bibitem{zhang2021distribution}
S.~Zhang, Z.~Li, S.~Yan, X.~He, and J.~Sun, ``Distribution alignment: A unified
  framework for long-tail visual recognition,'' in \emph{CVPR}, 2021.

\bibitem{zhong2021improving}
Z.~Zhong, J.~Cui, S.~Liu, and J.~Jia, ``Improving calibration for long-tailed
  recognition,'' in \emph{CVPR}, 2021.

\bibitem{zhao2021mgsvf}
H.~Zhao, Y.~Fu, M.~Kang, Q.~Tian, F.~Wu, and X.~Li, ``Mgsvf: Multi-grained slow
  vs. fast framework for few-shot class-incremental learning,'' \emph{TPAMI},
  2021.

\bibitem{cheraghian2021synthesized}
A.~Cheraghian, S.~Rahman, S.~Ramasinghe, P.~Fang, C.~Simon, L.~Petersson, and
  M.~Harandi, ``Synthesized feature based few-shot class-incremental learning
  on a mixture of subspaces,'' in \emph{ICCV}, 2021, pp. 8661--8670.

\bibitem{peng2022few}
C.~Peng, K.~Zhao, T.~Wang, M.~Li, and B.~C. Lovell, ``Few-shot
  class-incremental learning from an open-set perspective,'' in \emph{ECCV},
  2022.

\bibitem{shi2021overcoming}
G.~Shi, J.~Chen, W.~Zhang, L.-M. Zhan, and X.-M. Wu, ``Overcoming catastrophic
  forgetting in incremental few-shot learning by finding flat minima,'' in
  \emph{NeurIPS}, vol.~34, 2021, pp. 6747--6761.

\bibitem{yoon2020xtarnet}
S.~W. Yoon, D.-Y. Kim, J.~Seo, and J.~Moon, ``Xtarnet: Learning to extract
  task-adaptive representation for incremental few-shot learning,'' in
  \emph{ICML}, 2020.

\bibitem{chi2022metafscil}
Z.~Chi, L.~Gu, H.~Liu, Y.~Wang, Y.~Yu, and J.~Tang, ``Metafscil: A
  meta-learning approach for few-shot class incremental learning,'' in
  \emph{CVPR}, 2022, pp. 14\,166--14\,175.

\bibitem{zhou2022few}
D.-W. Zhou, H.-J. Ye, L.~Ma, D.~Xie, S.~Pu, and D.-C. Zhan, ``Few-shot
  class-incremental learning by sampling multi-phase tasks,'' \emph{TPAMI},
  2022.

\bibitem{zhu2021self}
K.~Zhu, Y.~Cao, W.~Zhai, J.~Cheng, and Z.-J. Zha, ``Self-promoted prototype
  refinement for few-shot class-incremental learning,'' in \emph{CVPR}, 2021,
  pp. 6801--6810.

\bibitem{ren2019incremental}
M.~Ren, R.~Liao, E.~Fetaya, and R.~Zemel, ``Incremental few-shot learning with
  attention attractor networks,'' in \emph{NeurIPS}, vol.~32, 2019.

\bibitem{joseph2022energy}
K.~Joseph, S.~Khan, F.~S. Khan, R.~M. Anwer, and V.~N. Balasubramanian,
  ``Energy-based latent aligner for incremental learning,'' in \emph{CVPR},
  2022, pp. 7452--7461.

\bibitem{lu2022geometer}
B.~Lu, X.~Gan, L.~Yang, W.~Zhang, L.~Fu, and X.~Wang, ``Geometer: Graph
  few-shot class-incremental learning via prototype representation,'' in
  \emph{KDD}, 2022.

\bibitem{yang2022dynamic}
B.~Yang, M.~Lin, Y.~Zhang, B.~Liu, X.~Liang, R.~Ji, and Q.~Ye, ``Dynamic
  support network for few-shot class incremental learning,'' \emph{TPAMI},
  2022.

\bibitem{weinan2022emergence}
W.~E and S.~Wojtowytsch, ``On the emergence of simplex symmetry in the final
  and penultimate layers of neural network classifiers,'' in \emph{Mathematical
  and Scientific Machine Learning}.\hskip 1em plus 0.5em minus 0.4em\relax
  PMLR, 2022, pp. 270--290.

\bibitem{lu2020neural}
J.~Lu and S.~Steinerberger, ``Neural collapse with cross-entropy loss,''
  \emph{arXiv preprint arXiv:2012.08465}, 2020.

\bibitem{zhu2021geometric}
Z.~Zhu, T.~DING, J.~Zhou, X.~Li, C.~You, J.~Sulam, and Q.~Qu, ``A geometric
  analysis of neural collapse with unconstrained features,'' in \emph{NeurIPS},
  2021.

\bibitem{ji2021unconstrained}
W.~Ji, Y.~Lu, Y.~Zhang, Z.~Deng, and W.~J. Su, ``An unconstrained layer-peeled
  perspective on neural collapse,'' in \emph{ICLR}, 2022.

\bibitem{mixon2020neural}
D.~G. Mixon, H.~Parshall, and J.~Pi, ``Neural collapse with unconstrained
  features,'' \emph{arXiv preprint arXiv:2011.11619}, 2020.

\bibitem{poggio2020explicit}
T.~Poggio and Q.~Liao, ``Explicit regularization and implicit bias in deep
  network classifiers trained with the square loss,'' \emph{arXiv preprint
  arXiv:2101.00072}, 2020.

\bibitem{zhou2022optimization}
J.~Zhou, X.~Li, T.~Ding, C.~You, Q.~Qu, and Z.~Zhu, ``On the optimization
  landscape of neural collapse under mse loss: Global optimality with
  unconstrained features,'' in \emph{ICML}, 2022.

\bibitem{tirer2022extended}
T.~Tirer and J.~Bruna, ``Extended unconstrained features model for exploring
  deep neural collapse,'' in \emph{ICML}, 2022.

\bibitem{thrampoulidis2022imbalance}
C.~Thrampoulidis, G.~R. Kini, V.~Vakilian, and T.~Behnia, ``Imbalance trouble:
  Revisiting neural-collapse geometry,'' in \emph{NeurIPS}, 2022.

\bibitem{behnia2023implicit}
T.~Behnia, G.~R. Kini, V.~Vakilian, and C.~Thrampoulidis, ``On the implicit
  geometry of cross-entropy parameterizations for label-imbalanced data,'' in
  \emph{AISTATS}.\hskip 1em plus 0.5em minus 0.4em\relax PMLR, 2023, pp.
  10\,815--10\,838.

\bibitem{zhong2023understanding}
Z.~Zhong, J.~Cui, Y.~Yang, X.~Wu, X.~Qi, X.~Zhang, and J.~Jia, ``Understanding
  imbalanced semantic segmentation through neural collapse,'' in \emph{CVPR},
  2023.

\bibitem{galanti2022on}
T.~Galanti, A.~Gy{\"o}rgy, and M.~Hutter, ``On the role of neural collapse in
  transfer learning,'' in \emph{ICLR}, 2022.

\bibitem{li2022principled}
X.~Li, S.~Liu, J.~Zhou, X.~Lu, C.~Fernandez-Granda, Z.~Zhu, and Q.~Qu,
  ``Principled and efficient transfer learning of deep models via neural
  collapse,'' \emph{arXiv preprint arXiv:2212.12206}, 2022.

\bibitem{xu2023dynamics}
M.~Xu, A.~Rangamani, Q.~Liao, T.~Galanti, and T.~Poggio, ``Dynamics in deep
  classifiers trained with the square loss: Normalization, low rank, neural
  collapse, and generalization bounds,'' \emph{Research}, vol.~6, p. 0024,
  2023.

\bibitem{du2022demystify}
Y.~Du, Y.~Yang, D.~Tao, and M.-H. Hsieh, ``Demystify problem-dependent power of
  quantum neural networks on multi-class classification,'' \emph{arXiv preprint
  arXiv:2301.01597}, 2022.

\bibitem{cheraghian2021semantic}
A.~Cheraghian, S.~Rahman, P.~Fang, S.~K. Roy, L.~Petersson, and M.~Harandi,
  ``Semantic-aware knowledge distillation for few-shot class-incremental
  learning,'' in \emph{CVPR}, 2021, pp. 2534--2543.

\bibitem{hoffer2018fix}
E.~Hoffer, I.~Hubara, and D.~Soudry, ``Fix your classifier: the marginal value
  of training the last weight layer,'' in \emph{ICLR}, 2018.

\bibitem{pernici2021regular}
F.~Pernici, M.~Bruni, C.~Baecchi, and A.~Del~Bimbo, ``Regular polytope
  networks,'' \emph{TNNLS}, 2021.

\bibitem{tanwisuth2021prototype}
K.~Tanwisuth, X.~Fan, H.~Zheng, S.~Zhang, H.~Zhang, B.~Chen, and M.~Zhou, ``A
  prototype-oriented framework for unsupervised domain adaptation,''
  \emph{NeurIPS}, vol.~34, pp. 17\,194--17\,208, 2021.

\bibitem{laenen2021on}
S.~Laenen and L.~Bertinetto, ``On episodes, prototypical networks, and few-shot
  learning,'' in \emph{NeurIPS}, 2021.

\bibitem{guo2022learning}
D.~dan Guo, L.~Tian, M.~Zhang, M.~Zhou, and H.~Zha, ``Learning
  prototype-oriented set representations for meta-learning,'' in \emph{ICLR},
  2022.

\bibitem{gidaris2018dynamic}
S.~Gidaris and N.~Komodakis, ``Dynamic few-shot visual learning without
  forgetting,'' in \emph{CVPR}, 2018.

\bibitem{wang2018cosface}
H.~Wang, Y.~Wang, Z.~Zhou, X.~Ji, D.~Gong, J.~Zhou, Z.~Li, and W.~Liu,
  ``Cosface: Large margin cosine loss for deep face recognition,'' in
  \emph{CVPR}, 2018.

\bibitem{kang2022class}
M.~Kang, J.~Park, and B.~Han, ``Class-incremental learning by knowledge
  distillation with adaptive feature consolidation,'' in \emph{CVPR}, 2022.

\bibitem{krizhevsky2009learning}
A.~Krizhevsky, G.~Hinton \emph{et~al.}, ``Learning multiple layers of features
  from tiny images,'' 2009.

\bibitem{deng2009imagenet}
J.~Deng, W.~Dong, R.~Socher, L.-J. Li, K.~Li, and L.~Fei-Fei, ``Imagenet: A
  large-scale hierarchical image database,'' in \emph{CVPR}, 2009.

\bibitem{russakovsky2015imagenet}
O.~Russakovsky, J.~Deng, H.~Su, J.~Krause, S.~Satheesh, S.~Ma, Z.~Huang,
  A.~Karpathy, A.~Khosla, M.~Bernstein \emph{et~al.}, ``Imagenet large scale
  visual recognition challenge,'' \emph{IJCV}, vol. 115, no.~3, pp. 211--252,
  2015.

\bibitem{wah2011caltech}
C.~Wah, S.~Branson, P.~Welinder, P.~Perona, and S.~Belongie, ``The caltech-ucsd
  birds-200-2011 dataset,'' 2011.

\bibitem{liu2022few}
H.~Liu, L.~Gu, Z.~Chi, Y.~Wang, Y.~Yu, J.~Chen, and J.~Tang, ``Few-shot
  class-incremental learning via entropy-regularized data-free replay,'' in
  \emph{ECCV}, 2022.

\bibitem{van2008visualizing}
L.~Van~der Maaten and G.~Hinton, ``Visualizing data using t-sne.'' \emph{JMLR},
  2008.

\bibitem{he2016deep}
K.~He, X.~Zhang, S.~Ren, and J.~Sun, ``Deep residual learning for image
  recognition,'' in \emph{CVPR}, 2016.

\bibitem{liu2020negative}
B.~Liu, Y.~Cao, Y.~Lin, Q.~Li, Z.~Zhang, M.~Long, and H.~Hu, ``Negative margin
  matters: Understanding margin in few-shot classification,'' in \emph{ECCV},
  2020, pp. 438--455.

\bibitem{zhang2020deepemd}
C.~Zhang, Y.~Cai, G.~Lin, and C.~Shen, ``Deepemd: Few-shot image classification
  with differentiable earth mover's distance and structured classifiers,'' in
  \emph{CVPR}, 2020, pp. 12\,203--12\,213.

\bibitem{yin2020dreaming}
H.~Yin, P.~Molchanov, J.~M. Alvarez, Z.~Li, A.~Mallya, D.~Hoiem, N.~K. Jha, and
  J.~Kautz, ``Dreaming to distill: Data-free knowledge transfer via
  deepinversion,'' in \emph{CVPR}, 2020, pp. 8715--8724.

\end{thebibliography}

%



%

\appendices

\section{Proof of Theorem \ref{theorem}}
\label{proof}

We consider the following problem,
\begin{align}
	\label{obj_app}
	\min_{\rmM^{(t)}}\quad & \frac{1}{N^{(t)}}\sum^{K^{(t)}}_{k=1}\sum^{n_k}_{i=1} \gL\left(\rvm^{(t)}_{k,i},\hat{\rmW}_{\rm ETF}\right), \ 0\le t \le T,  \\
	s.t. \quad & \| \rvm^{(t)}_{k,i} \|^2 \le 1, \quad \forall 1\le k \le K^{(t)},\ 1\le i \le n_k, \notag
\end{align}
where $\rvm^{(t)}_{k,i}\in\R^d$ denotes a feature variable that belongs to the $i$-th sample of class $k$ in session $t$, $n_k$ is number of samples in class $k$, $K^{(t)}$ is number of classes in session $t$, $N^{(t)}$ is the number of samples in session $t$, \emph{i.e.,} $N^{(t)}=\sum_{k=1}^{K^{(t)}}n_k$, and $\rmM^{(t)}\in\R^{d\times N^{(t)}}$ denotes a collection of $\rvm^{(t)}_{k,i}$. $\hat{\rmW}_{\rm ETF}\in\R^{d\times K}$ refers to the neural collapse terminus for the whole label space, where $K=\sum_{t=0}^TK^{(t)}$. 
We re-state the Theorem \ref{theorem} in our paper:
\begin{theorem*}
	Let $\hat{\rmM}^{(t)}$ denotes the global minimizer of Eq. (\ref{obj_app}) by optimizing the model incrementally from $t=0$, and we have $\hat{\rmM}=[\hat{\rmM}^{(0)}, \cdots, \hat{\rmM}^{(T)}]\in\R^{d\times \sum_{t=0}^T N^{(t)}}$. No matter if $\gL$ in Eq. (\ref{obj_app}) is CE or misalignment loss, for any column vector $\hat{\rvm}_{k,i}$ in $\hat{\rmM}$ whose class label is $k$, we have:
	\begin{equation}
		\label{theorem_eq_app}
		\|\hat{\rvm}_{k,i}\|=1,\ \ \hat{\rvm}^T_{k,i}\hat{\rvw}_{k'}=\frac{K}{K-1}\delta_{k,k'}-\frac{1}{K-1},
	\end{equation}
	for all $k, k'\in[1,K],\ 1\le i \le n_k$, where $K=\sum_{t=0}^TK^{(t)}$ denotes the total number of classes of the whole label space, $\delta_{k,k'}=1$ when $k=k'$ and 0 otherwise, and $\hat{\rvw}_{k'}$ is the class prototype in $\hat{\rmW}_{\rm ETF}$ for class $k'$. 
\end{theorem*}

\begin{proof}
From the definition of $\hat{\rmW}_{\rm ETF}$, we have,  
\begin{equation}
	\label{eq11}
	\hat{\rvw}^T_{k}\hat{\rvw}_{k'}=\frac{K}{K-1}\delta_{k,k'}-\frac{1}{K-1},\ \ \forall k, k'\in[1,K],
\end{equation}
where $\hat{\rvw}_{k}$ and $\hat{\rvw}_{k'}$ are two column vectors in $\hat{\rmW}_{\rm ETF}$. From Eq. (\ref{eq11}),
we have $\hat{\rmW}_{\rm ETF}\cdot\mathbf{1}_K=\mathbf{0}_d$, where $\mathbf{1}_K$ is an all-ones vector in $\R^K$, and $\mathbf{0}_d$ is an all-zeros vector in $\R^d$. Then we have,
\begin{equation}
	\label{sum_wk}
	\sum^K_{k=1} \hat{\rvw}_{k}=\mathbf{0}_d.
\end{equation}

When $\gL$ is the misalignment loss with the following definition, 
\begin{equation}
	\label{misalign}
	\mathcal{L}\left(\rvm^{(t)}_{k,i},\hat{\mathbf{w}}_{k}\right)=\frac{1}{2}\left(\hat{\mathbf{w}}^T_{k} \rvm^{(t)}_{k,i} - 1\right)^2,
\end{equation}
it is easy to identify that $\gL\ge0$ and the equality holds if and only if $ \hat{\rvw}_{k}^T\rvm^{(t)}_{k,i}=1$, for all $0\le t\le T, 1\le k\le K, 1\le i\le n_k$. Since $\|\hat{\rvw}_{k}\|=1$ and $\| \rvm^{(t)}_{k,i} \|^2 \le 1$, we have $ \hat{\rvw}_{k}^T\rvm^{(t)}_{k,i}\le1$. The equality holds if and only if $\| \rvm^{(t)}_{k,i} \|^2 = 1$ and $\cos\angle( \hat{\rvw}_{k}, \rvm^{(t)}_{k,i})=1$. Denote $\hat{\rmM}=[\hat{\rmM}^{(0)}, \cdots, \hat{\rmM}^{(T)}]\in\R^{d\times \sum_{t=0}^T N^{(t)}}$ as the global optimality of Eq. (\ref{obj_app}) for all sessions $0\le t\le T$. For any column vector $\hat{\rvm}_{k,i}$ in $\hat{\rmM}$, we have, 
\begin{equation}
	\|\hat{\rvm}_{k,i}\|=1,\ \ \hat{\rvm}^T_{k,i}\hat{\rvw}_{k'}=\frac{K}{K-1}\delta_{k,k'}-\frac{1}{K-1}, \notag
\end{equation}
for all $k, k'\in[1,K],\ 1\le i \le n_k$, which concludes the proof for the misalignment loss.

When $\gL$ is the cross-entropy (CE) loss, \emph{i.e.,}
\begin{equation}
	\label{ce}
	\gL\left(\rvm^{(t)}_{k,i},\hat{\rmW}_{\rm ETF}\right)=-\log \frac{\exp(\hat{\rvw}_{k}^T\rvm^{(t)}_{k,i})}{\sum_{j=1}^K\exp(\hat{\rvw}_{j}^T\rvm^{(t)}_{k,i})}, 
\end{equation}
where $0\le t \le T, 1\le k \le K^{(t)}$, and $1\le i \le n_k$. Since the problem is separable among $T+1$ sessions, we only analyze the $t$-th session and omit the superscript $(t)$ for simplicity. The objective in Eq. (\ref{ce}) is the sum of an affine function and log-sum-exp functions. When $\hat{\rmW}_{\rm ETF}$ is fixed, the loss is convex \emph{w.r.t} $\rvm_{k,i}$ with convex constraints. So, we can use the KKT condition for its global optimality. Based on Eq. (\ref{obj_app}) and Eq. (\ref{ce}), we have the Lagrange function,
\begin{align}
	\tilde{\gL} = & \frac{1}{N^{(t)}}\sum^{K^{(t)}}_{k=1}\sum^{n_k}_{i=1}    -\log \frac{\exp(\hat{\rvw}_{k}^T\rvm_{k,i})}{\sum_{j=1}^K\exp(\hat{\rvw}_{j}^T\rvm_{k,i})} + \notag \\ &\sum^{K^{(t)}}_{k=1}\sum^{n_k}_{i=1} \lambda_{k,i}(\| \rvm_{k,i}\|^2 -1),
\end{align}
where $\lambda_{k,i}$ is the Lagrange multiplier. The gradient with respect to $\rvm_{k,i}$ takes the form of:
\begin{equation}
	\frac{\partial \tilde{\gL} }{\partial \rvm_{k,i}}= -\frac{\left(1-p_k\right)}{N^{(t)}}\hat{\rvw}_{k}+\frac{1}{N^{(t)}}\sum_{j\ne k}^{K}p_j \hat{\rvw}_j+2\lambda_{k,i} \rvm_{k,i},
\end{equation}
where $1\le i\le n_k$, $1\le k\le K^{(t)}$, and $p_j$ is the softmax probability of $\rvm_{k,i}$ for the $j$-th class, \emph{i.e.,}
\begin{equation}
	p_j =  \frac{\exp(\hat{\rvw}_{j}^T\rvm_{k,i})}{\sum_{j'=1}^K\exp(\hat{\rvw}_{j'}^T\rvm_{k,i})}.
\end{equation}
Since $\| \hat{\rvw}_{k} \|=1$ and $\| \rvm_{k,i} \| \le1$, we have
$0<p_k<1$ for all $1\le k\le K$.

We now solve the equation $\frac{\partial \tilde{\gL} }{\partial \rvm_{k,i}}=0$. Assume that $\lambda_{k,i}=0$, and then we have,
\begin{equation}
	\label{lambda=0}
	\sum_{j\ne k}^{K}p_j \hat{\rvw}_j =  {\left(1-p_k\right)}\hat{\rvw}_{k}.
\end{equation}
Since $1-p_k=\sum^K_{j\ne k}p_j$ and Eq. (\ref{eq11}), multiplying $\hat{\rvw}_{k}$ by both sides of Eq. (\ref{lambda=0}), we have,
\begin{equation}
	\frac{K}{K-1}(1-p_k)=0,
\end{equation}
which contradicts with $0<p_k<1$. Then we have the other case $\lambda_{k,i}>0$. Based on the KKT condition, the global optimality $\hat{\rvm}_{k,i}$ satisfies that
\begin{equation}
	\label{eq19}
	\| \hat{\rvm}_{k,i} \|^2 =1.
\end{equation}
The equation $\frac{\partial \tilde{\gL} }{\partial \hat{\rvm}_{k,i}}=0$ leads to:
\begin{equation}
	\label{eq20}
	\sum_{j\ne k}^{K}p_j (\hat{\rvw}_j - \hat{\rvw}_k)+2N^{(t)}\lambda_{k,i} \hat{\rvm}_{k,i}=0.
\end{equation}
Based on Eq. (\ref{eq11}), for any $j'\ne k$, multiplying $\hat{\rvw}_{j'}$ by both sides of Eq. (\ref{eq20}), we have,
\begin{equation}
	p_{j'}\frac{K}{K-1} + 2N^{(t)}\lambda_{k,i}\hat{\rvm}_{k,i}^T \hat{\rvw}_{j'}=0.
\end{equation}
Since $\forall k\in [1,K]$, $p_k>0$, we have $\hat{\rvm}_{k,i}^T \hat{\rvw}_{j'}<0$. Then for any $j_1,j_2\ne k$, 
\begin{equation}
	\label{eq22}
	\frac{p_{j_1}}{p_{j_2}}=\frac{\exp(\hat{\rvw}_{j_1}^T\hat{\rvm}_{k,i})}{\exp(\hat{\rvw}_{j_2}^T\hat{\rvm}_{k,i})}=\frac{ \hat{\rvw}_{j_1}^T\hat{\rvm}_{k,i}}{ \hat{\rvw}_{j_2}^T \hat{\rvm}_{k,i}}.
\end{equation}
The function $f(x)=\exp(x)/x$ is monotonically increasing when $x<1$. So, Eq. (\ref{eq22}) indicates that
\begin{equation}
	p_{j_1}=p_{j_2},\ \hat{\rvw}_{j_1}^T\hat{\rvm}_{k,i} = \hat{\rvw}_{j_2}^T \hat{\rvm}_{k,i},\ \forall j_1, j_2 \ne k,
\end{equation}
and
\begin{equation}
	\label{eq24}
	p_j = \frac{1-p_k}{K-1}=-\frac{2N^{(t)}\lambda_{k,i}\hat{\rvm}_{k,i}^T \hat{\rvw}_{j}(K-1)}{K},\ \forall j\ne k.
\end{equation}
Multiplying $\hat{\rvw}_{k}$ by both sides of Eq. (\ref{eq20}), we have,
\begin{equation}
	\label{eq25}
	-\frac{K}{K-1}(1-p_k) + 2N^{(t)}\lambda_{k,i}\hat{\rvm}_{k,i}^T \hat{\rvw}_{k}=0.
\end{equation}
Combing Eq. (\ref{eq24}) and Eq. (\ref{eq25}), we have, 
\begin{equation}
	\label{eq26}
	\hat{\rvm}_{k,i}^T \hat{\rvw}_{j}(K-1) + \hat{\rvm}_{k,i}^T \hat{\rvw}_{k}=0,\ \forall j\ne k.
\end{equation}
Based on $p_j=\frac{1-p_k}{K-1}, \forall j\ne k$, and Eq. (\ref{sum_wk}), we can rewrite Eq. (\ref{eq20}) as:
\begin{equation}
	-\frac{(1-p_k)K}{K-1}\hat{\rvw}_{k}+2N^{(t)}\lambda_{k,i} \hat{\rvm}_{k,i}=0,
\end{equation}
which means that $\hat{\rvm}_{k,i}$ is aligned with $\hat{\rvw}_{k}$, \emph{i.e.,} $\cos\angle(\hat{\rvm}_{k,i}, \hat{\rvw}_{k})=1$. Given that $\| \hat{\rvw}_{k} \|=1$ and 	$\| \hat{\rvm}_{k,i} \| =1$ (Eq. (\ref{eq19})), we have,
\begin{equation}
	\hat{\rvm}_{k,i}^T \hat{\rvw}_{k}=1,\notag
\end{equation}
and Eq. (\ref{eq26}) leads to:
\begin{equation}
	\hat{\rvm}_{k,i}^T \hat{\rvw}_{j}=-\frac{1}{K-1},\ \forall j\ne k.\notag
\end{equation}
Therefore, for any column vector $\hat{\rvm}_{k,i}$ in $\hat{\rmM}$, we have, 
\begin{equation}
	\|\hat{\rvm}_{k,i}\|=1,\ \ \hat{\rvm}^T_{k,i}\hat{\rvw}_{k'}=\frac{K}{K-1}\delta_{k,k'}-\frac{1}{K-1}, \notag
\end{equation}
for all $k, k'\in[1,K],\ 1\le i \le n_k$, which concludes the proof for the CE loss.
\end{proof}

\begin{figure*}[t]
	\begin{subfigure}{.33\textwidth}
		\centering
		\includegraphics[width=1.\linewidth]{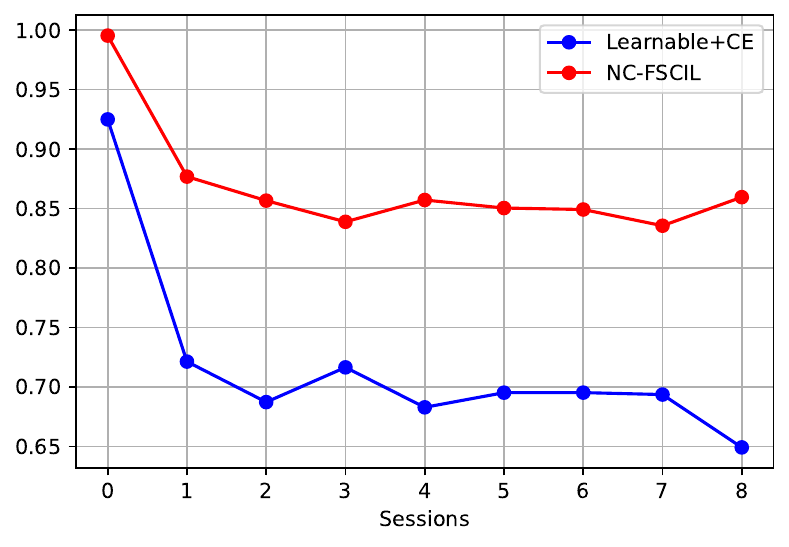}  
		\caption{train (each)}
		\label{hwintra:sub-first}
	\end{subfigure}
	\begin{subfigure}{.33\textwidth}
		\centering
		\includegraphics[width=1.\linewidth]{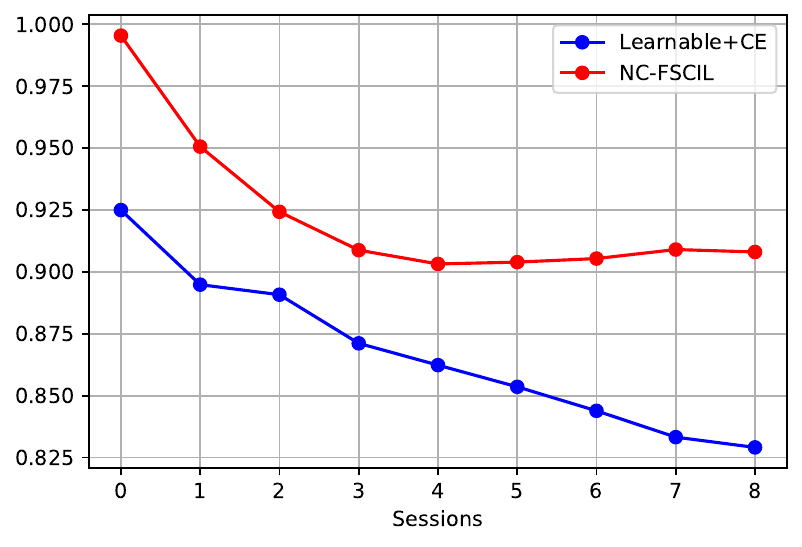}  
		\caption{train (accumulate)}
		\label{hwintra:sub-second}
	\end{subfigure}
	\begin{subfigure}{.33\textwidth}
		\centering
		\includegraphics[width=1.\linewidth]{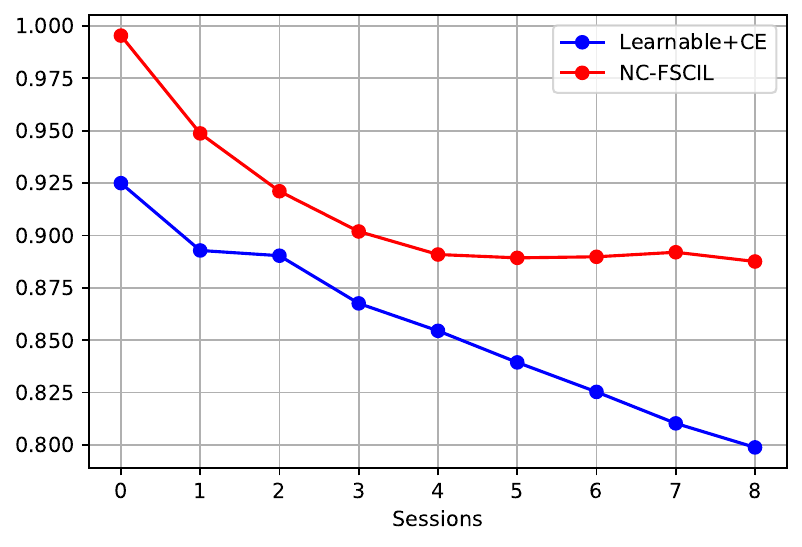}  
		\caption{train (base)}
		\label{hwintra:sub-three}
	\end{subfigure}
	\begin{subfigure}{.33\textwidth}
		\centering
		\includegraphics[width=1.\linewidth]{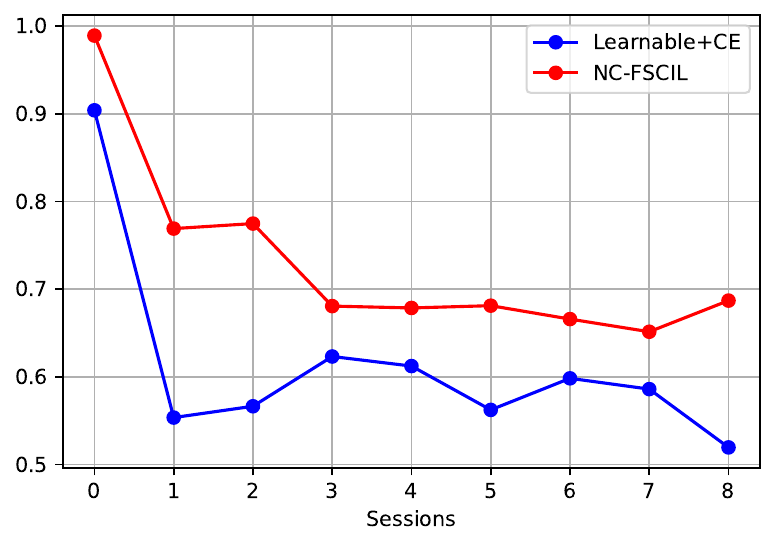}  
		\caption{test (each)}
		\label{hwintra:sub-four}
	\end{subfigure}
	\begin{subfigure}{.33\textwidth}
		\centering
		\includegraphics[width=1.\linewidth]{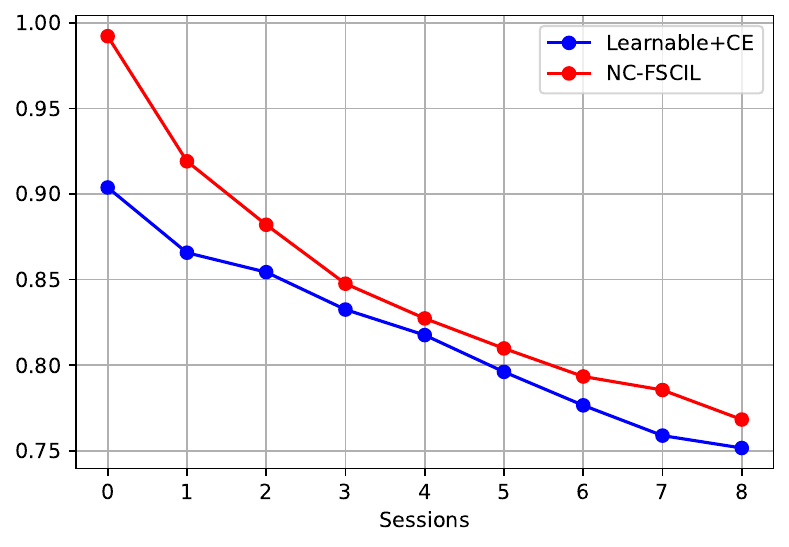} 
		\caption{test (accumulate)}
		\label{hwintra:sub-five}
	\end{subfigure}
	\begin{subfigure}{.33\textwidth}
		\centering
		\includegraphics[width=1.\linewidth]{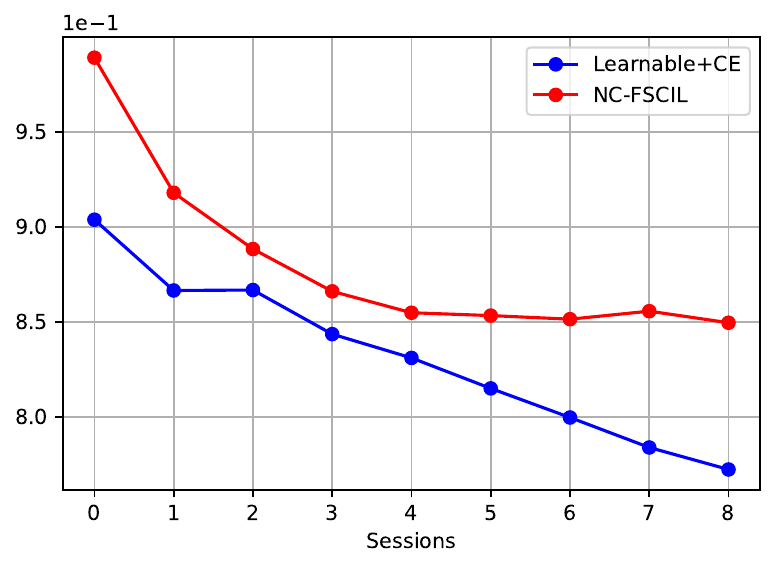}  
		\caption{test (base)}
		\label{hwintra:sub-six}
	\end{subfigure}
	\vspace{-3mm}
	\caption{Average cosine similarity between features and classifier prototypes of the same class, \emph{i.e.,} $\mathrm{Avg}_{k}\{\cos\angle(\rvm_{k}-\rvm_{G}, {\rvw}_{k}) \}$, where $\rvm_{k}$ is the within-class mean of class $k$ features, $\rvm_{G}$ denotes the global mean, and ${\rvw}_{k}$ is the classifier prototype of class $k$. Statistics are performed among classes in each session (a and d), all encountered classes by the current session (b and e), and only the base session classes (c and f), on train set (a, b, c) and test set (d, e, f), for our NC-FSCIL on miniImageNet.}
	\label{fig:hwintra}
\end{figure*}

\begin{figure*}[t]
	\begin{subfigure}{.33\textwidth}
		\centering
		\includegraphics[width=1.\linewidth]{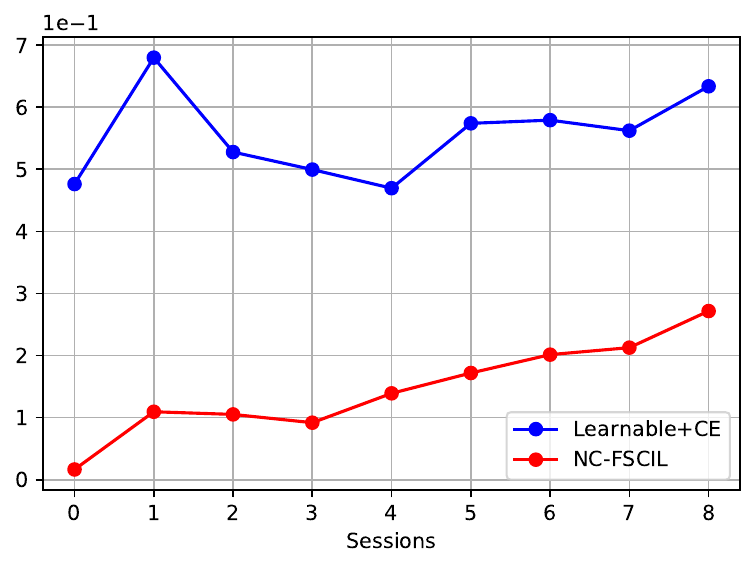}  
		\caption{train (each)}
		\label{trace:sub-first}
	\end{subfigure}
	\begin{subfigure}{.33\textwidth}
		\centering
		\includegraphics[width=1.\linewidth]{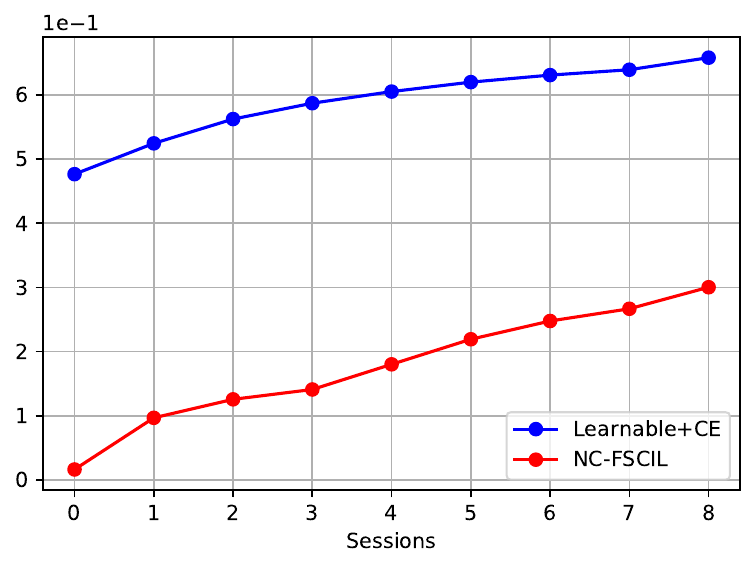}  
		\caption{train (accumulate)}
		\label{trace:sub-second}
	\end{subfigure}
	\begin{subfigure}{.33\textwidth}
		\centering
		\includegraphics[width=1.\linewidth]{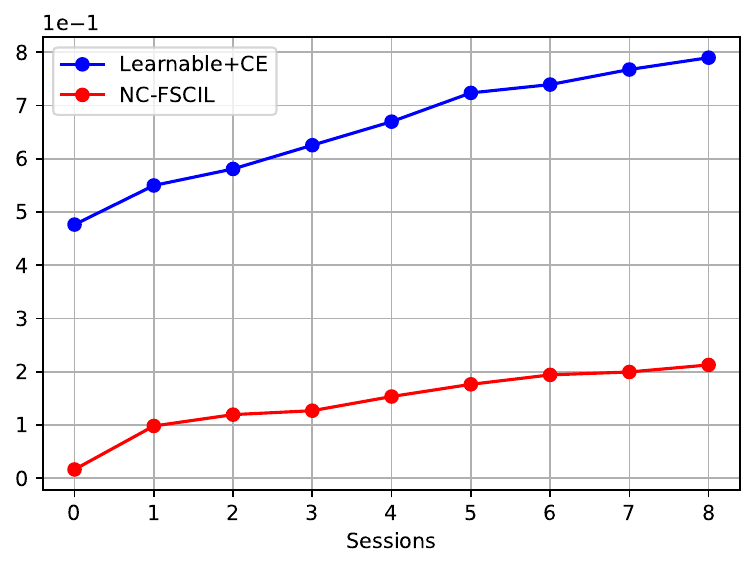}  
		\caption{train (base)}
		\label{trace:sub-three}
	\end{subfigure}
	\begin{subfigure}{.33\textwidth}
		\centering
		\includegraphics[width=1.\linewidth]{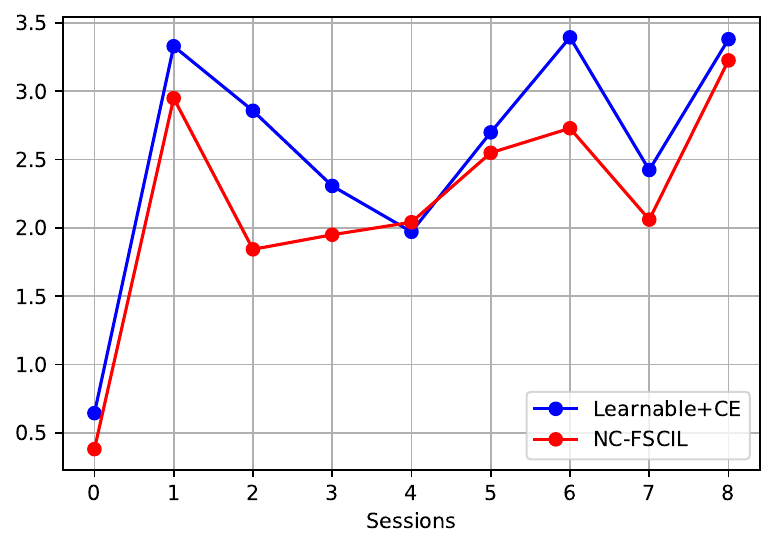}  
		\caption{test (each)}
		\label{trace:sub-four}
	\end{subfigure}
	\begin{subfigure}{.33\textwidth}
		\centering
		\includegraphics[width=1.\linewidth]{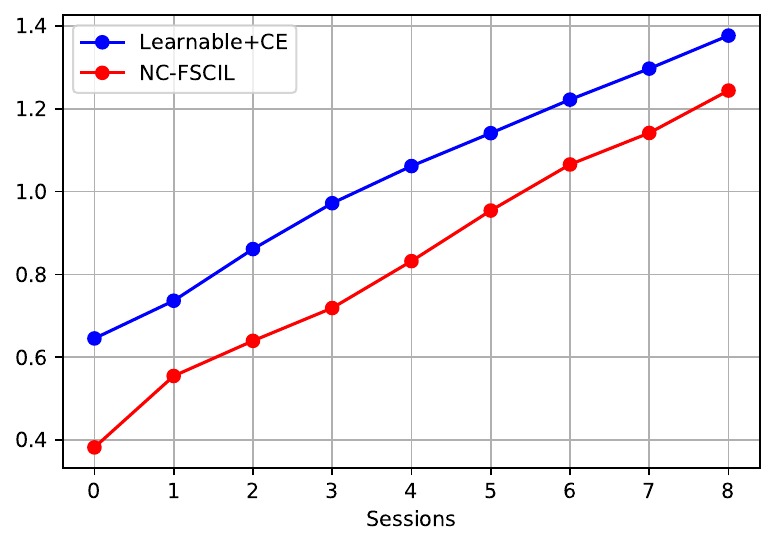} 
		\caption{test (accumulate)}
		\label{trace:sub-five}
	\end{subfigure}
	\begin{subfigure}{.33\textwidth}
		\centering
		\includegraphics[width=1.\linewidth]{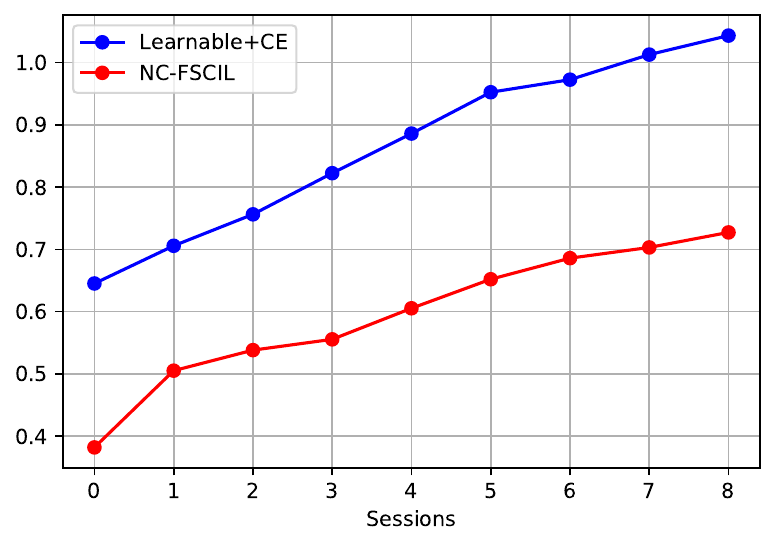}  
		\caption{test (base)}
		\label{trace:sub-six}
	\end{subfigure}
	\vspace{-3mm}
	\caption{Trace ratio of within-class covariance to between-class covariance, \emph{i.e.,} ${{\rm tr}(\Sigma_W)}/{{\rm tr}(\Sigma_B)}$, where $\Sigma_W$ is the within-class covariance, and $\Sigma_B$ denotes the between-class covariance. Statistics are performed among classes in each session (a and d), all encountered classes by the current session (b and e), and only the base session classes (c and f), on train set (a, b, c) and test set (d, e, f), for our NC-FSCIL on miniImageNet.}
	\label{fig:trace}
\end{figure*}

\section{Training Details}

\subsection{CIL, LTCIL, and GCIL}


For ImageNet (including ImageNet-100, ImageNet-1k, and ImageNet100-LT datasets), we adopt the basic ResNet-18~\cite{he2016deep} as backbone, which is the same as previous methods~\cite{wang2022foster, shi2022mimicking, hou2019learning, yan2021dynamically}. For CIFAR (including CIFAR-100 and CIFAR100-LT), we adopt a shallower ResNet-12~\cite{he2016deep} as backbone. The learning rates for CIFAR and ImageNet are 0.25 and 0.1, respectively. For both CIFAR and ImageNet (including the 5 datasets of CIL and LTCIL in total), we train for 200 epochs in each session using the SGD optimizer with a momentum of 0.9, a weight decay of 5e-4, and a batchsize of 512. We use the cosine annealing learning rate scheduler (minimum learning rate ratio of 0.01), and use Mixup and Cutmix for augmentations following common practices. 
The generalized case (GCIL) experiments are conducted under the same training setting as CIL and LTCIL, and all methods use the same backbone network of ResNet-18.

\subsection{FSCIL}

\label{train details}

Prior studies widely adopt ResNet-12, ResNet-18, and ResNet-20 \cite{he2016deep} for FSCIL experiments. 
For miniImageNet and CIFAR-100, we use ResNet-12 following \cite{hersche2022constrained}. For CUB-200, we use ResNet-18 (pre-trained on ImageNet) following most studies. We adopt a two-layer MLP block as the projection layer following the practice in \cite{peng2022few}.
We use the standard data pre-processing and augmentation schemes including random resizing, random flipping, and color jittering \cite{tao2020few,zhang2021few,peng2022few}. We train all models with a batchsize of 512 in the base session, and a batchsize of 64 (containing new session data and intermediate features in the memory) in each incremental session. On miniImageNet, we train for 500 epochs in the base session, and 100-170 iterations in each incremental session. The initial learning rate is 0.25 for base session, and 0.025 for incremental sessions. On CIFAR-100, we train for 200 epochs in the base session, and 50-200 iterations in each incremental session. The initial learning rate is 0.25 for both base and incremental sessions. On CUB-200, we train for 80 epochs in the base session, and 105-150 iterations in each incremental session. The initial learning rates are 0.025 and 0.05 for base session and incremental sessions, respectively. In all experiments, we adopt a cosine annealing learning rate strategy, and use SGD with momentum as optimizer.

\section{More Results}
\label{more results}

\subsection{More results on FSCIL}

Our experimental result on CUB-200 is shown in Table \ref{table:cub}. We achieve a better accuracy in the last session than most of the baseline methods. Although we do not surpass ALICE in the last session on CUB-200, we still have the best average incremental accuracy among all methods.

We also visualize the average cosine similarities between feature and prototype of the same class, \emph{i.e.,} $\mathrm{Avg}_{k}\{\cos\angle(\rvm_{k}-\rvm_{G}, {\rvw}_{k}) \}$, in our NC-FSCIL, and the baseline method with a learnable classifier and the CE loss.
A higher average $\cos\angle(\rvm_{k}-\rvm_{G}, {\rvw}_{k} )$ indicates that feature centers are more closely aligned with their corresponding classifier prototypes of the same class.
As shown in Figure \ref{fig:hwintra}, the values of our method are consistently higher than those of the baseline method. Figure \ref{hwintra:sub-first} and Figure \ref{hwintra:sub-four} reveal that our method has a better feature-classifier alignment in each session of the incremental training on both train and test sets. When we measure on all the encountered classes by each session in Figure \ref{hwintra:sub-second} and Figure \ref{hwintra:sub-five}, the metric for our method does not change obviously after the 4-th session on train set, while the metric for the baseline method keeps decreasing as training incrementally. Especially for the base session classes, our method is able to keep the metric stable on both train and test sets after the decline of the first 3-4 sessions, as shown in Figure \ref{hwintra:sub-three} and Figure \ref{hwintra:sub-six}. As a comparison, the baseline method cannot mitigate the deterioration. Given that the base session has the most classes, the performance on base session classes largely decides the final accuracy after the last session for FSCIL. Therefore, the superiority of our method can be attributed to our ability of keeping the feature-classifier alignment well for the old classes, especially for the base session classes.

We also visualize the trace ratio of within-class covariance to between-class covariance of the last-layer features, \emph{i.e.}, ${{\rm tr}(\Sigma_W)}/{{\rm tr}(\Sigma_B)}$.
The within-class covariance $\Sigma_W $ and the between-class covariance $\Sigma_B$ are defined as:
\begin{equation}
	\Sigma_W = \mathrm{Avg}_k\{\Sigma^{(k)}_W\},\quad\Sigma^{(k)}_W=\mathrm{Avg}_{i}\{(\rvm_{k,i}-\rvm_{k})(\rvm_{k,i}-\rvm_{k})^T\},\notag
\end{equation}
and
\begin{equation}
	\Sigma_B = \mathrm{Avg}_{k}\{(\rvm_{k}-\rvm_{G})(\rvm_{k}-\rvm_{G})^T\},\notag
\end{equation}
where $\rvm_{k,i}$ is the feature of sample $i$ in class $k$, $\rvm_{k}$ is the within-class mean of class $k$ features, and $\rvm_{G}$ denotes the global mean of all features. A lower within-class variation with a higher between-class variation corresponds to a better Fisher Discriminant Ratio. As shown in Figure \ref{fig:trace}, we compare the trace ratio of within-class covariance to between-class covariance between our method and the baseline method. We observe similar patterns to Figure \ref{fig:hwintra}. Concretely, the trace ratio metric of our method is consistently lower that of baseline. For the base session classes, the metric of our method increases more mildly, which corroborates our ability of maintaining the performance on the old classes, 
and is in line with the indications from Figure \ref{fig:hwintra} and the results in main paper.

\subsection{Detailed setups on the generalized case}

In our generalized case, each incremental session could by any of the normal, long-tail, and few-shot cases by sampling. We list the sampled setups of our results in Table \ref{tab:generalized_setup}. The run time consumed by all incremental sessions of all methods trained under the same setting is reported in Table \ref{tab:generalized_time}.

\begin{table*}[t] 
	\renewcommand\arraystretch{1}
	\begin{center}
		\caption{Performance of FSCIL in each session on CUB-200 and comparison with other studies. The top rows list class-incremental learning and few-shot learning results implemented by \cite{tao2020few,zhang2021few,liu2022few,zhou2022forward} in the FSCIL setting. ``Average Acc.'' is the average incremental accuracy. ``Final Improv.'' calculates the improvement of our method after the last session. }
		\resizebox{\textwidth}{!}{
			\begin{tabular}{lccccccccccccc}
				\toprule
				\multicolumn{1}{l}{\multirow{2}{*}{\bf Methods}} & \multicolumn{11}{c}{\bf Accuracy in each session (\%) $\uparrow$} & \bf Average & \bf Final \\ 
				\cmidrule{2-12}
				& \bf 0   & \bf 1      &\bf  2      & \bf 3    & \bf 4     &\bf  5  & \bf 6     & \bf 7      & \bf 8     &   \bf 9  & \bf 10   & \bf Acc. &\bf  Improv.  \\ 
				\midrule
				iCaRL~\cite{rebuffi2017icarl}        & 68.68   & 52.65     & 48.61    & 44.16  & 36.62   & 29.52  & 27.83  & 26.26    & 24.01   &23.89  & 21.16  & 36.67 &\bf +38.28   \\
				EEIL~\cite{castro2018end}         & 68.68   & 53.63      &47.91    & 44.20  & 36.30     & 27.46  & 25.93  & 24.70     & 23.95   &24.13 & 22.11  & 36.27&\bf +37.33   \\
				NCM~\cite{hou2019learning}         & 68.68   & 57.12     & 44.21    & 28.78  & 26.71    & 25.66  & 24.62  & 21.52    & 20.12   &20.06  & 19.87  & 32.49&\bf +39.57    \\
				Fixed classifier~\cite{pernici2021class}          & 68.47   & 51.00     &45.42    & 40.76  &   35.90   & 33.18  & 27.23  & 24.24     & 21.18   &17.34 & 16.20  & 34.63 & \bf +43.24 \\
				D-NegCosine~\cite{liu2020negative}    &74.96  & 70.57     & 66.62   & 61.32   & 60.09    & 56.06  & 55.03  & 52.78     & 51.50   &50.08 & 48.47   & 58.86 &\bf +10.97 \\
				D-DeepEMD~\cite{zhang2020deepemd}     & 75.35   & 70.69     & 66.68   & 62.34  & 59.76     & 56.54  & 54.61  & 52.52    &50.73   &49.20 & 47.60 & 58.73 & \bf +11.84  \\
				D-Cosine~\cite{vinyals2016matching}    &75.52  & 70.95     & 66.46   & 61.20   & 60.86    & 56.88  & 55.40  & 53.49     & 51.94   & 50.93 & 49.31   & 59.36 & \bf +10.13     \\
				DeepInv \cite{yin2020dreaming} & 75.90&	70.21&	65.36&	60.14 &	58.79&	55.88&	53.21&	51.27&	49.38&	47.11&	45.67&	57.54&\bf +13.77\\
				\midrule
				TOPIC~\cite{tao2020few}            & 68.68   & 62.49      & 54.81    & 49.99   & 45.25     & 41.40 & 38.35  & 35.36    & 32.22  &28.31 & 26.28   & 43.92 &\bf +33.16 \\
				IDLVQ \cite{chen2021incremental}  &77.37&	74.72&	70.28&	67.13&	65.34&	63.52&	62.10&	61.54&	59.04&	58.68&	57.81&	65.23& \bf +1.63\\
				SPPR~\cite{zhu2021self}  & 68.68   & 61.85     & 57.43    & 52.68   & 50.19   & 46.88 & 44.65   & 43.07      & 40.17     & 39.63   & 37.33&49.32 & \bf +22.11      \\
				\cite{cheraghian2021synthesized} & 68.78&	59.37&	59.32&	54.96&	52.58&	49.81&	48.09&	46.32&	44.33&	43.43&	43.23&	51.84& \bf +16.21\\
				CEC~\cite{zhang2021few}                & 75.85   & 71.94    & 68.50   & 63.50   & 62.43    & 58.27 & 57.73 & 55.81    &54.83  &53.52  & 52.28   & 61.33 &\bf +7.16   \\
				LIMIT \cite{zhou2022few}  &76.32&	74.18&	72.68&	69.19&\bf	68.79&\bf	65.64&	63.57&	62.69&	\bf 61.47&	60.44&	58.45&	66.67& \bf +0.99\\
				MgSvF \cite{zhao2021mgsvf}	&72.29&	70.53&	67.00&	64.92&	62.67&	61.89&	59.63&	59.15&	57.73&	55.92&	54.33&	62.37& \bf +5.11\\
				MetaFSCIL \cite{chi2022metafscil}	&75.9	&72.41	&68.78	&64.78	&62.96	&59.99	&58.3	&56.85	&54.78	&53.82	&52.64	&61.93& \bf +6.8\\
				FACT \cite{zhou2022forward}	&75.90	&73.23	&70.84	&66.13	&65.56	&62.15	&61.74	&59.83	&58.41	&57.89	&56.94	&64.42& \bf +2.5\\
				Data-free replay \cite{liu2022few}	&75.90	&72.14	&68.64	&63.76	&62.58	&59.11	&57.82	&55.89	&54.92	&53.58	&52.39	&61.52& \bf +7.05 \\
				ALICE \cite{peng2022few}	&77.40	&72.70	&70.60	&67.20	&65.90	&63.40	&62.90	&61.90	&60.50	&\bf 60.60	&\bf 60.10	&65.75& -0.66\\
				\midrule
				\bf  NC-FSCIL (ours)            & \bf 80.45   & \bf 75.98     & \bf72.30    & \bf 70.28  & 68.17   & 65.16  & \bf 64.43   &\bf 63.25     & 60.66   & 60.01   & 59.44 & \bf 67.28 & \\
				\bottomrule
			\end{tabular}
		}
		\label{table:cub}
	\end{center}
\end{table*}

\begin{table*}[!ht]
	\renewcommand\arraystretch{1}
	\centering
	\small
	\caption{The detailed sampled setups for each experiment of the generalized case results. ``nm'' denotes the normal case, ``lt'' denotes the long-tail case, and ``fs'' refers to the few-shot case.}
	\resizebox{\textwidth}{!}
	{
		\begin{tabular}{c|c|ccccccccccccccccccccccccc}
			\toprule[0.1em]
			
			\multirow{2}{*}{Mode} &  \multirow{2}{*}{Exp No.} & \multicolumn{25}{c}{Session Types}\\
			
			&&1&2&3&4&5&6&7&8&9&10&11&12&13&14&15&16&17&18&19&20&21&22&23&24&25\\
			\midrule
			\multirow{4}{*}{{10 Steps}} & exp-1 & nm & fs & nm & lt & fs & fs & nm & nm & fs & lt & 
			\multicolumn{15}{c}{-}\\
			
			& exp-2 & lt & lt & nm & lt & nm & nm & lt & fs & fs & nm &  \multicolumn{15}{c}{-}\\
			& exp-3 & nm & nm & fs & fs & lt & fs & lt & nm & nm & nm &  \multicolumn{15}{c}{-}\\
			& exp-4 & nm & lt & nm & nm & fs & fs & nm & lt & lt & nm &  \multicolumn{15}{c}{-}\\
			
			\midrule		
			
			\multirow{4}{*}{{25 Steps}} & exp-5 & nm &	nm & lt	&fs	& lt & fs	& fs &	nm	& lt	& lt	& nm &	fs&	lt&	nm&	fs&	lt&	lt&	lt&	fs&	lt&	nm&	lt&	nm	& fs &	lt \\
			
			& exp-6 &fs &	nm &	fs &	nm &	lt &	nm &	fs &	nm &	fs &	fs &	lt &	 nm &	nm &	lt &	nm &	nm &	nm &	fs &	fs &	lt &	lt &	nm &	 nm &	fs &	nm \\
			& exp-7 & nm &	fs &	nm &	nm &	lt &	lt &	fs &	nm &	fs &	fs &	nm &	 lt &	fs &	fs &	lt &	fs &	fs &	lt &	nm &	nm &	lt &	lt &	fs &	nm &	 fs  \\
			& exp-8 & nm &	nm &	lt &	nm &	fs &	lt &	lt &	nm &	fs &	lt &	lt &	fs &	 lt &	lt &	nm &	lt &	nm &	nm &	nm &	lt &	lt &	lt &	lt &	lt &	 fs  \\
			\bottomrule[0.1em]
		\end{tabular}
	}
	\label{tab:generalized_setup}
\end{table*}

\begin{table*}
	\renewcommand\arraystretch{1}
	\centering
	\small
	\caption{Run time (GPU hours) consumed by all incremental sessions tested on an NVIDIA-V100 node.}
	\setlength{\tabcolsep}{3pt}
	{
		\begin{tabular}{l|cccccccc}
			\toprule[0.1em]
			
			Method & exp-1 & exp-2 & exp-3 &  exp-4 & exp-5 & exp-6 & exp-7 & exp-8\\			
			\midrule		
			LUCIR & 1.1 & 1.2 & 1.8 & 2.0 & 8.0 & 8.0 & 7.3 & 8.0 \\
			AANet & 1.4 & 1.2 & 1.8 & 2.0 & 8.0 & 8.0 & 7.8 & 8.3 \\
			LT-CIL & 2.0 & 2.2 & 3.3 & 3.3 & 13.8 & 15.3 & 13.8 & 15.3 \\
			ALICE & 0.3 & 0.3 & 0.3 & 0.3 & 1.5 & 1.5 & 1.5 & 1.5 \\
			NC-GCIL (ours) & 1.1 & 1.2 & 1.8 & 2.0 & 8.0 & 8.0 & 7.3 & 8.0\\
			\bottomrule[0.1em]
		\end{tabular}
	}
	\label{tab:generalized_time}
\end{table*}

\end{document}